\documentclass[11pt,reqno]{article}

\usepackage[top=2.5cm, bottom=2.5cm, left=3.2cm, right=3.2cm]{geometry}	

\usepackage{amsmath, amsfonts, bm, amssymb, mathtools, bm, commath, amsthm}
\usepackage{doi,appendix}
\usepackage{color, scalerel}
\usepackage{cleveref}
\usepackage[dvipsnames]{xcolor}
\usepackage[font=small,labelfont=bf]{caption}
\usepackage{enumitem}
\usepackage{chngcntr}
\usepackage{bm} 
\usepackage{algorithmic,algorithm}
\usepackage{cancel}
\usepackage[backend=bibtex,natbib,style=authoryear,maxcitenames=2,maxbibnames=99]{biblatex}
\theoremstyle{plain}
\newtheorem{theorem}{Theorem}
\newtheorem{proposition}[theorem]{Proposition}
\newtheorem{lemma}[theorem]{Lemma}
\newtheorem{assumption}{Assumption}
\newtheorem{definition}{Definition}
\newtheorem{corollary}[theorem]{Corollary}
\newtheorem{example}{Example}
\theoremstyle{remark}

\newtheorem{algo}{Algorithm}

\renewcommand{\cal}[1]{\mathcal{#1}}

\renewcommand{\r}{\mathbb{R}}

\newcommand{\n}{\mathbb{N}}

\newcommand{\iprod}[2]{\left\langle{#1},{#2}\right\rangle}
\usepackage[symbol]{footmisc}
\newcommand{\Pacd}{\cal{P}_2^{1}(\r^{d_x})}
\newcommand{\Macd}{\cal{M}_2^{1}}
\newcommand{\KL}{\mathsf{KL}}

\DeclareMathOperator*{\argmax}{arg\,max}
\DeclareMathOperator*{\argmin}{arg\,min}

\usepackage{xspace}
\newcommand{\PLI}{P{\L}I\xspace}

\hypersetup{
	colorlinks   = true, 
	urlcolor     = {green!50!black}, 
	linkcolor    = {green!50!black}, 
	citecolor   = {green!50!black} 
}

\title{Fast convergence of the Expectation Maximization algorithm under a logarithmic Sobolev inequality} 

\author{Rocco Caprio and Adam M. Johansen\\Department of Statistics, University of Warwick.\medskip\\ Email: \texttt{\{\href{mailto:rocco.caprio@warwick.ac.uk}{rocco.caprio}, \href{mailto:a.m.johansen@warwick.ac.uk}{a.m.johansen}\}@warwick.ac.uk.}}

\begin{document}
	\maketitle
	\begin{abstract}
		We present a new framework for analysing the Expectation Maximization (EM) algorithm. Drawing on recent advances in the theory of gradient flows over Euclidean--Wasserstein spaces, we extend techniques from alternating minimization in Euclidean spaces to the EM algorithm, via its representation as coordinate-wise minimization of the free energy.	
        In so doing, we obtain finite sample error bounds and exponential convergence of the EM algorithm under a natural generalisation of the log-Sobolev inequality. We further show that this framework naturally extends to several variants of EM, offering a unified approach for studying such algorithms.
		\\
		
	\end{abstract}
	
	\section{Introduction} \label{sec:intro}

\subsection{The Expectation Maximization Algorithm}
Consider the problem of fitting a probabilistic model, with Lebesgue density $p_\theta(x,y)$, featuring latent variables, $x\in\r^{d_x}$, to observed data, $y$. Within the maximum likelihood framework one seeks model parameters, $\theta\in\r^{d_\theta}$, that maximize the (marginal) likelihood, i.e. the probability density, $p_\theta(y)$, of observing the obtained data upon integrating out the latent variables: $\theta_\star$ in
\begin{equation} \label{eq:MLE}
	\cal{O}_\star:= \argmax_{\theta\in\r^{d_\theta}} p_\theta(y)= \argmax_{\theta\in \r^{d_\theta}} \int p_\theta(x,y)d x.
\end{equation}
One can also characterize the conditional distribution of the latent variables given the observed data under the model specified by parameter vector $\theta_\star$ as
$p_{\theta_\star}(x|y):= {p_{\theta_\star}(x,y)}/{p_{\theta_\star}(y)},$
and due to the Bayesian flavour of this computation and the connection with empirical Bayes \citep{Robbins1955} we will term this the \emph{posterior} distribution of $x$ throughout.
For most models of practical interest, the integral in \eqref{eq:MLE} is intractable, we have no closed-form expressions for $p_\theta(y)$ or its derivatives, and we are unable to directly optimize $p_\theta(y)$. The most popular algorithm to solve this problem is the Expectation Maximization (EM) algorithm \citep{Dempster1977}, which is the object of analysis of this work, together with some of its traditional, or more novel, variations.
As the data, $y$, can be treated as fixed, in the the remainder we shorten our notation and write $\rho_\theta(x)$ for $p_\theta(x,y)$, $Z_\theta$ for its normalizing constant $p_\theta(y)$, $\pi_\theta(x) = \rho_\theta(x) / Z_\theta$ for $p_\theta(x|y)$, and additionally $\ell(\theta,x)$ for the (complete) log-likelihood $\log(\rho_\theta(x))$. When dealing with measures absolutely continuous with respect to the Lebesgue measure we use the same symbol to refer to both the measure and its Lebesgue density.

In this paper, we utilize the connection between EM and a coordinate-wise minimization algorithm applied to the free energy functional identified by \citet{Neal1998}, and illustrated below, to provide non-asymptotic error bounds for EM algorithms under an extended form of the \emph{log-Sobolev inequality}. To do this, we extend an argument commonly used to understand Euclidean alternating minimization algorithms by comparison with gradient descent via the descent lemma \citep{Beck2013,Beck2015,Both2022}, together with recently developed results for using and understanding gradients on the product of Euclidean and Wasserstein spaces \citep{PGD,Caprio2025}. This results in a somewhat different approach to earlier literature, and it directly connects EM's exponential convergence with the concepts of gradients in the space of probability distributions, appropriate log-Sobolev inequalities, and other results and techniques in optimal transport and sampling. 

Such analysis is quite direct, requires limited further technical work, and yields state-of-the-art conclusions for the fast global convergence of EM. We believe that the extended log-Sobolev condition characterizes well this convergence regime, and we provide examples of toy hierarchical models where it holds. Naturally, many practical applications of EM fall without this model class and so outside the immediate reach of the arguments in this paper. We believe that one of the most promising aspects of this analysis is that it strongly suggests that the arguments employed in the literature on functional inequalities to investigate slower-than-exponential and local convergence regimes will also be applicable, upon suitable (non-trivial) generalization, to the EM setting. These have the potential to provide a theory of non-asymptotic convergence, which is not currently available. We highlight some particularly promising directions in Section~\ref{sec:discussion}.

\subsection{Relevant literature}  The importance of EM in statistics, machine learning and science more broadly, 
means that its convergence properties have been extensively studied. Beyond the characterization provided by \citet{Dempster1977}, early analysis was provided by \citet{Boyles1983,Wu1983}. These and other works \citep{Hero1995,Meng1994,Mclachlan2007} focus on asymptotic convergence rates. However, relatively little is known about the non-asymptotic performance.
\citet{Kumar2017,Kunstner2021} leverage interesting connections between the EM algorithm, generalized surrogate optimization methods and mirror descent, respectively, to study this problem. \cite{Kumar2017} show global sub-exponential  rates for the parameter estimates under the hypothesis of strong concavity of the \textit{surrogate function},
\begin{equation} \label{eq:surrogatef}
	\theta\mapsto Q(\theta|\theta') = \int\ell(\theta,x) \pi_{\theta'}(dx).
\end{equation}
for all $\theta'\in\r^{d_\theta}$. Leveraging connections with mirror descent,
\cite{Kunstner2021} obtain sub-exponential rates for the posterior estimates' relative entropy when $(\rho_\theta(d x))_{\theta\in\r^{d_\theta}}$ forms a minimal exponential family (which implies that the surrogate is concave), and exponential rates if further the ratio of missing information is bounded; \cite{Aubin2022} obtain similar results for the latent EM algorithm. 
\cite{Balakrishnan2017} study the related but different problem of deriving convergence rates to the true population parameter, and not the maximum likelihood estimate, in the case of an infinite sample (`population EM') and of a finite sample (`sample EM'), by  also considering an assumption of strong concavity of the surrogate, and provided one starts the algorithm in some neighbourhood of the optimum.
Other lines of work have focused on more specific models or situations, such as Gaussian mixtures \citep{Xu2016,Hao2018,Wang2015}, the case of misspecified models \citep{Dwivedi2020}, or stochastic EM methods \citep{Karimi2019}.

\subsection{A differential analysis of the EM Algorithm}
The EM algorithm, in its standard description, is formulated as the iterative maximization procedure on the surrogate function
\begin{equation} \label{eq:traditionalem}
	\theta_{k+1}=\argmax_{\theta\in \r^{d_\theta}} Q(\theta|\theta_k), \quad k\geq 0.
\end{equation}
Albeit it is common, for both the analysis and the implementation itself of EM, to focus only, and directly, on the evolution of the parameter estimates $\theta_k$ via the surrogate \eqref{eq:traditionalem}, in this work we recognize that the iterations \eqref{eq:traditionalem} are, actually, equivalent to two separate operations, \textit{each} of which is seen to minimize the same function. This is the key observation of \citet{Neal1998}.
\begin{algo} \label{alg:em}
	EM. Input: initial values $(\theta_0,q_0)$ 
	\begin{tabbing}
		\qquad \enspace For $k\geq 0$ \\
		\qquad \qquad  Update the parameter estimate $\theta_{k+1} = \argmax_\theta \int \ell(\theta;x)q_{k}(dx)$ \\
		\qquad \qquad Update the posterior estimate $	q_{k+1} = \pi_{\theta_{k+1}}$ 
	\end{tabbing}
\end{algo}  
Let $\cal{P}(\r^{d_x})$ denote the space of Borel probability measures on the Euclidean space $\r^{d_x}$.
\begin{proposition} \label{pr:emiscd}
	The steps of the EM iterations are equivalent to
	\begin{align*}
		\theta_{k+1} \in& \argmin_{\theta\in\r^{d_\theta}} F(\theta,q_{k}),&
		q_{k+1} \in& \argmin_{q\in\cal{P}(\r^{d_x})} F(\theta_{k+1},q),
	\end{align*}
\end{proposition}
\noindent 
where $F:\mathbb{R}^{d_\theta}\times \mathcal{P}(\mathbb{R}^{d_x}) \to \mathbb{R}$ is the \textit{free energy} functional:
\begin{equation}
	\label{eq:entropy}
	F( \theta, q):=\left\{\begin{array}{cl} 
		\displaystyle\int \log\left(\frac{q(x)}{\rho_\theta(x)}\right) q(d x)&\text{if } q\ll \rho_\theta(d x)\\ +\infty&\text{otherwise}\end{array}\right. \quad \forall (\theta, q) \in \cal{M}.
\end{equation}
Therefore, EM algorithm is equivalent to an alternating (or coordinate-wise) minimization procedure on $\cal{M}:=\r^{d_\theta}\times\cal{P}(\r^{d_x})$ for the free energy. This result suggests that, rather than focussing on the evolution of $\theta_{k}$ only as per \eqref{eq:traditionalem}, one should instead look at the joint convergence of $(\theta_k,q_k)$ towards $F$'s minimizers $(\theta_\star,\pi_{\theta_\star})$, and the related evolution in terms of free energy. This turns out to be a fruitful approach, which we now describe.

For reasons that will become clear later, we focus our attention on $\cal{M}_2:=\r^{d_\theta}\times \cal{P}_2(\r^{d_x})$, where $\cal{P}_2(\r^{d_x})$ is the restriction of $\cal{P}(\r^{d_x})$ to elements which are absolutely continuous with respect to Lebesgue measure and admit finite second moments. Since EM is a minimization procedure on $F$, it is key to the analysis to understand $F$'s variations along the EM iterates. Proposition \ref{pr:emiscd} shows that the free energy can only decrease along the EM iterations:
\begin{equation*}
	F(\theta_0,q_0) \geq 	F(\theta_1,q_0) \geq 	F(\theta_1,q_1) \geq \dots \geq F(\theta_k,q_k) \geq F(\theta_{k+1},q_k) \geq F(\theta_{k+1},q_{k+1}),
\end{equation*}
and, to characterize EM's exponential convergence, we would like to quantify this decrease.
The alternating minimization representation of EM suggests we might be able to adapt the analysis of alternating minimization algorithms on Euclidean space \citep{Beck2013,Beck2015,Both2022}. When minimizing a smooth function $f:\mathbb{R}^{d_x} \to \mathbb{R}$ satisfying 
\begin{equation*}
	2\lambda(f(x)-f_\star) \leq \norm{\nabla_x f(x)}^2 \quad \forall x\in\r^{d_x}, \quad \text{where} \quad \lambda>0 \quad \text{and} \quad \quad f_\star := \inf f,
\end{equation*}
a gradient growth condition known as the Polyak--{\L}ojasiewicz inequality, this analysis can be conducted by comparing the alternating minimization updates with appropriate gradient steps. Since the free energy is a function on $\cal{M}_2$, it is not immediately clear how to translate this approach to our setting. However, recent advances in optimal transport provide a solution.

In particular, we show that under a smoothness assumption, we can lower bound EM's free energy decrease abstractly in terms of the norm of $\textup{grad}_{\cal{M}_2}F$, $F$'s gradient in the geometry induced by the product of the Euclidean and Wasserstein metrics on $\cal{M}_2$ (see Appendix \ref{app:geom} for details on this).
Because the posterior updates is on the space of probability measures, the notion of a gradient step with which to compare the EM step relates to the concepts of Wasserstein gradient flows.
Having lower-bounded the decrease in free energy in terms of $\textup{grad}_{\cal{M}_2}F$ via smoothness, we assume that $\textup{grad}_{\cal{M}_2}F$'s norm grows at least quadratically away from $F$'s minimizers along EM iterations, a natural analogue of the \PLI on $\cal{M}_2$. More precisely:
\begin{equation}\label{eq:extlogsobolev}
	2 \lambda (F (\theta, q) - F_\star) \leq  \norm{\textup{grad}_{\cal{M}_2}F(\theta,q)}_{\cal{M}_2}^2
\end{equation}		
for all relevant EM iterates $(\theta,q)$, where $\lambda>0$ is a positive constant and where, as also detailed in Appendix \ref{app:geom}, we can compute
\begin{align} 
	\label{eq:fisherinformation}
	I (\theta, q) := \norm{\textup{grad}_{\cal{M}_2}F(\theta,q)}_{\cal{M}_2}^2
	=& \norm{\int \nabla_\theta \ell ( \theta, x ) q(d x)}^2 + \int \norm{\nabla_x \log \left( \frac{q(x)}{\rho_\theta(x)}\right)}^2 q(d x).
\end{align}
\begin{align*}
	F_\star := \inf_{( \theta, q ) \in \cal{M}_2} F( \theta, q ) =\inf_{\theta\in\r^{d_\theta}}F(\theta,\pi_\theta)= -\log \bigg( \sup_{\theta \in \r^{d_\theta}} Z_\theta \bigg) =:-\log Z_\star.
\end{align*}

We refer to $I(\theta,q)$ as the \textit{extended Fisher information functional},  because if the parameter space is the trivial space $\{\theta\}$, this quantity reduces to the relative Fisher information functional (e.g. see equation (8) in \cite{Otto2000}):
\begin{align} \label{eq:stdfisherinformation}
	I(q||\pi_\theta) := \int \norm{\nabla_x\log\left(\frac{q(x)}{\pi_\theta(x)}\right)}^2 q(d x).
\end{align}
\begin{definition}[Extended log-Sobolev inequality]\label{def:LSI} The measures $(\rho_\theta (d x))_{\theta \in \r^{d_\theta}}$ satisfy the extended log-Sobolev inequality with constant $\lambda > 0$ if \eqref{eq:extlogsobolev} holds for all $(\theta, q)\in\Macd:=\r^{d_\theta}\times\Pacd$.	
\end{definition}
$\Pacd$ is the restriction of $\cal{P}_2(\r^{d_x})$ to probability measures having at least a.e.~differentiable densities to ensure \eqref{eq:fisherinformation} is well defined.
The extended log-Sobolev inequality is a generalization of the log-Sobolev and Polyak--{\L}ojasiewicz inequalities of optimal transport and optimization, and was studied in \cite{Caprio2025} in the context of a gradient flow on $\cal{M}_2$, where it appears naturally.  While the log Sobolev inequality is a statement about a single probability distribution, and \PLI about a single function, the extended log-Sobolev inequality is a statement about a statistical model. In particular, it is always verified if the model is strongly log-concave, and it might hold for non-concave models. The Appendix details sufficient conditions for it to hold, with examples of Bayesian hierarchical models.

Once convergence of the free energy of EM iterates to $F_\star$ via the extended log-Sobolev inequality has been established, we investigate fast convergence of the EM iterates themselves $(\theta_k,q_k)$ to $\cal{M}_\star := \argmin F( \theta, q ) = \{(\theta_\star,\pi_{\theta_\star}):\theta_\star\in\cal{O}_\star\}$, i.e.~to a (local) maximum of the marginal likelihood and its corresponding posterior. To do so, we shall consider the inequality
\begin{equation} \label{eq:exttalagrand}
	2 (F (\theta, q) - F_\star) \geq \lambda \mathsf{d} (( \theta, q ),\cal{M}_\star)^2,
\end{equation}
where $\mathsf{d}$ is the product metric
\begin{align}\label{eq:ddeterministic}
	\mathsf{d}(( \theta, q ),(\theta',q')):=\sqrt{\mathsf{d}_E(\theta,\theta')^2+\mathsf{d}_{W_2}(q,q')^2},
\end{align}
with $\mathsf{d}_E$ and $\mathsf{d}_{W_2}$ denoting the Euclidean and Wasserstein-2 distances, 
which extends a well-known inequality by \citet{Talagrand1996}  and the quadratic growth condition \citep{Anitescu2000}:
\begin{definition}[Extension of the Talagrand inequality]\label{def:talagrand} 
	The measures $(\rho_\theta(d x))_{\theta\in\r^{d_\theta}}$ satisfy an extension of the Talagrand inequality with constant $\lambda > 0$ if \eqref{eq:exttalagrand} holds for all $(\theta, q)\in\cal{M}_2$. 
\end{definition}
An natural extension of a theorem of \citet{Otto2000} shows that there is a deep link between the extended log-Sobolev inequality and the extension of the Talagrand inequality. To state the result, we impose some regularity conditions on the model.
\begin{assumption}\label{ass:model} (i) $\ell$ is twice continuously differentiable in both arguments, and $\nabla^2 \ell (\theta,x) \preceq \iota I_{d_\theta+d_x}$ for some $\iota\in\r$;  (ii) for all $\theta$ in $\r^{d_{\theta}}$ and $x$ in $\r^{d_x}$, $\rho_\theta(x)>0$; (iii) $\pi_\theta$ has finite second moments for all $\theta\in\r^{d_\theta}$; (iv) $Z_\theta<\infty$ for any $\theta\in\r^{d_\theta}$ (v) $\pi_\theta$'s second moments are uniformly bounded in $\cal{O}_\star$.  
\end{assumption}
Assumption \ref{ass:model} imposes only mild regularity conditions on the model and is often satisfied by models used in practice (notice that here $\iota$ is allowed to be any real number). Condition (v) is only required later by Lemma \ref{lemma:ptintegralisnotbig} and Theorem \ref{thm:Femconvergence}, and if $\cal{O}_\star$ consists of finitely many points (e.g.~if Assumption \ref{ass:strongconcave} below holds), it is redundant.
\begin{theorem}[Theorem 2 in \cite{Caprio2025}]\label{thm:exttalagrand}
	If Assumption~\ref{ass:model} holds, and the measures $(\rho_\theta(d x))_{\theta\in\r^{d_\theta}}$ satisfy the extended log-Sobolev inequality with constant $\lambda$, then they also satisfy the extension of the Talagrand inequality with the same constant. 
\end{theorem}
Appendix \ref{sec:checkingass} shows that Assumption~\ref{ass:model} ensures that the conditions in \cite{Caprio2025} hold. For our analysis, this theorem tells us that it is enough to establish convergence of the free energy of EM iterates under the extended log-Sobolev inequality, as the extension of the Talagrand inequality to translates this into the convergence of the iterates themselves. 
As a result, we are able to derive non-asymptotic convergence error bounds in $\mathsf{d}$-distance for both the posterior and parameter estimates of EM, and some of its variants, by assuming only a smoothness condition and that the model satisfies the extended log-Sobolev inequality \eqref{eq:extlogsobolev}.

\section{Non-asymptotic analysis of EM and related algorithms} \label{sec:emconv}

\subsection{EM Algorithm} \label{sec:em}

The goal is to establish exponential convergence of the EM iterates under the extended log-Sobolev inequality. We first establish the exponential convergence of the free energy to its minimizer and then transfer the result, via the extension of the Talagrand inequality and Theorem \ref{thm:exttalagrand}, to the EM iterates themselves. 
Proposition \ref{pr:emiscd} shows that the free energy can only decrease along the EM iterations, and in order to prove EM's exponential convergence, we need to quantify the decrease. To do so, we compare EM's free energy decreases due to the E-M steps with those induced by appropriate gradient-based algorithms.
Since we think of the extended Fisher information functional $I$ \eqref{eq:fisherinformation}  as the squared norm of the free energy gradient, under a smoothness assumption we can lower bound the magnitude of this decrease in terms of $I$. 
\begin{assumption} \label{ass:gradLip}
	 $\ell$ is differentiable and its gradient $\nabla \ell$ is (i) $L_\theta$-Lipschitz in $\theta$ uniformly in $x$; 
	(ii) $L_x$-Lipschitz in $x$ uniformly in $\theta$.
\end{assumption}

\begin{lemma} \label{lemma:Fdecr} If Assumption \ref{ass:gradLip}(i) holds, for all $k\in\n$,
	\begin{equation} \label{eq:F_em_theta_lb}
		F(\theta_k,q_k)-F(\theta_{k+1},q_k) \geq \frac{1}{2L_\theta}I(\theta_k,q_k)
	\end{equation}
\end{lemma}
\begin{proof}
	Consider  $\vartheta_k:=\theta_k+(1/L_\theta)\int \nabla_\theta\ell(\theta_k,x)q_k(d x)$. By the minimality of $\theta_{k+1}$,
	\begin{equation} \label{eq:F_em_theta_lb_1}
		F(\theta_k,q_k) - F(\theta_{k+1},q_k) \geq F(\theta_k,q_k) - F(\vartheta_{k},q_k),
	\end{equation}
	and because $\theta\mapsto \nabla_\theta\ell(\theta,x)$ is $L_\theta$-Lipschitz for all $x\in\r^{d_x}$, 
	\begin{equation*} \label{eq:F_em_theta_lb_2}
		\ell(\theta_k,x) -\ell(\vartheta_k,x) +\iprod{\nabla_\theta \ell(\theta_k,x)}{\vartheta_k-\theta_k}\leq   \frac{L_\theta}{2}\norm{\vartheta_k-\theta_k}^2, 
	\end{equation*}
	so that 
	\begin{align} \label{eq:descentoneuclideansp}
		&F(\theta_k,q_k) - F(\vartheta_k,q_k) 
		= \int (\ell(\vartheta_k,x)-\ell(\theta_k,x))q_k(d x) \nonumber \\
		&\geq \int \left(\frac{1}{L_\theta}\iprod{\nabla_\theta \ell(\theta_k,x)}{\int \nabla_\theta\ell(\theta_k,z)q_k(d z)}-\frac{1}{2L_\theta}\norm{\int \nabla_\theta\ell(\theta_k,z)q_k(d z)}^2\right)q_k(d x)  \nonumber \\
		&= \frac{1}{2L_\theta}\norm{\int\nabla_\theta\ell(\theta_k,x)q_k(d x)}^2,
	\end{align}
	(an inequality known as the \textit{descent lemma} in optimization). Now, since
	\begin{equation} \label{eq:qisminimizing}
		q_{k} \in \argmin_{q\in\cal{P}(\r^{d_x})} F(\theta_{k},q) = \pi_{\theta_k} \Rightarrow I(\theta_k,q_k) = \norm{\int\nabla_\theta\ell(\theta_k,x)q_k(d x)}^2
	\end{equation} 
    the claim follows upon combining the last two equations with \eqref{eq:F_em_theta_lb_1}. 
\end{proof}
Now we just need to know that, along EM iterations, $F$'s gradient grows at least quadratically away from the set of minimizers, but that is precisely what the extended log-Sobolev inequality says.   
\begin{proposition} \label{pr:Femconvergencetheta}
	Let Assumption \ref{ass:gradLip}(i) hold, and assume that the measures $(\rho_\theta(d x))_{\theta\in\r^{d_\theta}}$ satisfy the extended log-Sobolev inequality with constant $\lambda>0$. Then, for all $k\in\n$,
	\begin{equation*}
		F(\theta_{k},q_{k})-F_\star \leq (1-\lambda/L_\theta)^{k}(F(\theta_0,q_0)-F_\star)\leq e^{-k\lambda/L_\theta}(F(\theta_0,q_0)-F_\star).
	\end{equation*}
\end{proposition}
\begin{proof}
	By Lemma \ref{lemma:Fdecr} and then the extended log-Sobolev inequality, 
	\begin{equation*}
		F(\theta_k,q_k)-F(\theta_{k+1},q_{k+1}) \geq F(\theta_k,q_k) - F(\theta_{k+1},q_{k}) \geq \frac{1}{2L_\theta} I(\theta_k,q_k) \geq \frac{\lambda}{L_\theta}(F(\theta_k,q_k)-F_\star)
	\end{equation*}
	for all $k\in\n$, and the claim follows upon rearranging this inequality and iterating.
\end{proof}
We can use the extension of Talagrand's inequality provided by Theorem \ref{thm:exttalagrand} to conclude convergence in $\mathsf{d}$-distance of the EM iterates. 
\begin{corollary} \label{cor:emconvergencetheta}
	Under Assumptions \ref{ass:model} and \ref{ass:gradLip}(i),
	\begin{equation*}
		\lambda \mathsf{d}((\theta_k,q_k),\cal{M}_\star)^2 \leq 2(1-\lambda/L_\theta)^{k}(F(\theta_0,q_0)-F_\star)\leq 2e^{-k\lambda /L_\theta}(F(\theta_0,q_0)-F_\star).
	\end{equation*}
	
\end{corollary}

Having demonstrated that differential arguments provide an efficient way to obtain a convergence bound for both the parameter and the posterior estimates in EM to the maximum likelihood estimate,
we will sharpen this bound before commenting upon it.

To prove the above results, we considered and quantified only the decrease in free energy due to the parameter updates. We can improve the bound by also considering the decreases due posterior updates.  To do so, we would like to follow the same principles as the proofs above, by comparing EM's updates to appropriate gradient steps. We need a notion of a gradient step for $q\mapsto F(\theta_{k+1},q)\propto\KL(q||\pi_{\theta_{k+1}})$. As is well-known in some areas of optimal transport and sampling, in the Wasserstein-2 geometry this coincides with a Langevin step, in the sense that following this gradient direction coincides with evolving the probability density according to the  overdamped Langevin diffusion's Fokker--Planck equation (we included some details in Appendix \ref{app:geom}). Hence, as an analogue of \eqref{eq:F_em_theta_lb_1}, we wish to bound
\begin{equation} \label{eq:F_em_x_lb_1} 
	F(\theta_{k+1},q_k)-F(\theta_{k+1},q_{k+1}) \geq   F(\theta_{k+1},q_k)-F(\theta_{k+1},\textup{Law}(X_k+h\nabla_x \ell(\theta_{k+1},X_k) + \sqrt{2h}\xi_k))
\end{equation}
from below, where $\textup{Law}(X_k)=q_k$, $h>0$ is to be chosen based upon $L_x$, and $\xi_k$ is a standard normal random variable. We would like to obtain a lower bound on the right hand side quantity that depends on $L_x$ and the extended Fisher information functional, giving something that looks like \eqref{eq:F_em_theta_lb}. The only difficulty is that there is not an immediate candidate for an analogue of the descent lemma (the inequality \eqref{eq:descentoneuclideansp}) on $\cal{P}_2(\r^{d_x})$. However, we have the following:

\begin{lemma}[Descent lemma on $\cal{P}_2(\r^{d_x})$] \label{lemma:descentonP} Let Assumption \ref{ass:gradLip}(ii) hold. Let $(p_t)$ be an interpolation in $\cal{P}_2(\r^{d_x})$ between $q_k$ at $t=kh$ and $\textup{Law}(X_k+h\nabla_x \ell(\theta_{k+1},X_k) + \sqrt{2h}\xi_k)$, where $\textup{Law}(X_k)=q_k$, at $t=(k+1)h$, defined by the law of
	\begin{align*}
		Z_{t_-} =& X_k & d Z_t &= \nabla_x \ell(\theta_{k+1},Z_{t_-}) d t + \sqrt{2} d W_t,
	\end{align*}
	where $t\in [kh,(k+1)h]$ and $t_-=kh$. If $h\leq 1/4L_x$,
	\begin{equation} \label{eq:descentonP}
		\partial_t \KL(p_t||\pi_{\theta_{k+1}}) \leq -\frac{1}{2}I(p_t||\pi_{\theta_{k+1}}) + 6L_x^2d_x(t-t_-).
	\end{equation}
\end{lemma}
This result has been used to study Langevin Monte Carlo and was established in \citet{Vempala2019} and refined in \citet[Section 4.2]{Chewi2023}. Appendix \ref{app:pgdconv} contains  a generalization of this result, which we need to study a variation of EM later, and whose proof also illustrates how the above lemma can be established. The term $6L_x^2d_x(t-t_-)$ is a bias term, which arises because of the non-smoothness of the relative entropy \citep{Wibisono2018}.

We need one last ingredient: in the proof of Proposition \ref{pr:Femconvergencetheta} we used the fact that $q_k$ minimizes $F(\theta_{k+1},\cdot)$ to obtain \eqref{eq:qisminimizing}. Now, we can use the fact that $\theta_{k+1}$ minimizes $F(\cdot,q_k)$, to obtain 
\begin{equation} \label{eq:thetaisminimizing}
	\theta_{k+1} \in \argmin_{\theta\in\r^{d_\theta}} F(\theta,q_k) \Rightarrow \int \nabla_\theta \ell(\theta_{k+1},x)q_k(d x) = 0 \Rightarrow  I(\theta_{k+1},q_k)=I(q_k||\pi_{\theta_{k+1}}).
\end{equation} 
However, since we will work with the interpolation as in Lemma \ref{lemma:descentonP} we also need to know that $\int \nabla_\theta \ell(\theta_{k+1},x)p_t(d x)$ is small and establishing that is the focus of the next lemma.
\begin{lemma} \label{lemma:ptintegralisnotbig}
	Let Assumption \ref{ass:gradLip}(ii) hold and let $(p_t)$ be as in Lemma \ref{lemma:descentonP}. For $h\leq 1/4L_x$,
	\begin{equation*}
		\norm{\int \nabla_\theta \ell(\theta_{k+1},x)p_t(d x)}^2 \leq L_x(t-t_-)(C+4d_xL_x), \qquad t\in [kh,(k+1)h] 
	\end{equation*}
	with $C:=\sup_k \frac{1}{2}E\left\{\norm{\nabla_x \ell(\theta_{k+1},X_k)}^2\right\}$. In particular, when $q_k$ is  the EM update $q_k = \arg\min_{q\in \cal{P}_2(\r^{d_x})} F(\theta_{k+1},q)$, and $(\theta_\dagger,x_\dagger)$ is a stationary point of $\ell$, there holds
	\begin{equation} \label{eq:ptintegralisnotbig}
		C\leq L^2\sup_{\theta_\star\in\cal{O}_\star}\left\{\int \norm{x-x_\dagger}^2\pi_{\theta_\star}(d x)+\norm{\theta_\star-\theta_\dagger}^2 \right\} +  \frac{2L^2}{\lambda}(F(\theta_0,q_0)-F_\star).
	\end{equation}
\end{lemma} 
\eqref{eq:ptintegralisnotbig} shows that $C$ is finite under Assumption \ref{ass:model}. It is common to shift coordinates so that $(\theta_\dagger,x_\dagger)=(0,0)$. We now improve on Proposition \ref{pr:Femconvergencetheta}.

\begin{theorem} \label{thm:Femconvergence}
	Let Assumptions \ref{ass:model} and \ref{ass:gradLip} hold, and assume that the measures $(\rho_\theta(d x))_{\theta\in\r^{d_\theta}}$ satisfy the extended log-Sobolev inequality with constant $\lambda>0$. Then, for all $k\in\n$,
	\begin{equation*} \label{eq:Femconvergence}
		F(\theta_{k},q_{k})-F_\star \leq \inf \bigg\{ e^{-\lambda(k/L_\theta + \sum_{j=1}^{k} h_j)}(F(\theta_0,q_0)-F_\star) + B\sum_{j=1}^{k}h^2_{j} e^{-\lambda( \sum_{i=j+1}^{k} h_i+(k-j)/L_\theta)} \bigg\}
	\end{equation*}
	where the infimum is taken across all sequences $\{h_k\}$ such that $h_k\leq 1/4L_x$ for all $k\in\n$, and where $B:=(8L_x^2d_x+CL_x/2)$ and $C$ is as in Lemma \ref{lemma:ptintegralisnotbig}. If only Assumption \ref{ass:gradLip}(i) holds, the bound of Proposition \ref{pr:Femconvergencetheta} still holds.
\end{theorem}
\begin{proof}
	Because  $\partial_t F(\theta_{k+1},p_t) = 	\partial_t \KL(p_t||\pi_{\theta_{k+1}})$, we can write, by equation \eqref{eq:descentonP}, the extended log-Sobolev inequality and Lemma \ref{lemma:ptintegralisnotbig}, that for all $h_{k+1}\leq1/4L_x$, $t\in[kh_{k+1},(k+1)h_{k+1}]$, with $t_- = kh_{k+1}$:
	\begin{align*} \label{eq:F_em_x_lb_2} 
		\partial_t F(\theta_{k+1},p_t) &\leq -\frac{1}{2}I(\theta_{k+1},p_t) +  \frac{1}{2}\norm{\int \nabla_\theta\ell(\theta_{k+1},x)p_t(d x)}^2 + 6L_x^2d_x(t-t_-)   \nonumber	\\
		&\leq -\lambda (F(\theta_{k+1},p_t)-F_\star) +(t-t_-)B.
	\end{align*}
	Hence, 
	\begin{equation*}
		\partial_t \left\{e^{t\lambda}(F(\theta_{k+1},p_{(k+1)h_{k+1}})-F_\star)\right\} = e^{t\lambda}\left\{\lambda(F(\theta_{k+1},p_t)-F_\star)+\partial_t F(\theta_{k+1},p_t)\right\} \leq e^{t\lambda}h_{k+1}B.
	\end{equation*}
	Integrating between $t=kh_{k+1}$ and $t=(k+1)h_{k+1}$ we obtain
	\begin{equation*}
		e^{(k+1)h_{k+1}\lambda}(F(\theta_{k+1},p_{(k+1)h_{k+1}})-F_\star)-e^{k h_{k+1}\lambda}(F(\theta_{k+1},p_{kh_{k+1}})-F_\star) \leq e^{(k+1)h_{k+1}\lambda}h^2_{k+1}B, 
	\end{equation*}
	and thus:
	\begin{equation} \label{eq:F_em_x_lb_3}
		F(\theta_{k+1},p_{(k+1)h_{k+1}})-F_\star \leq e^{-h_{k+1}\lambda}(F(\theta_{k+1},q_k)-F_\star) + h^2_{k+1}B.
	\end{equation}
	Combining the left hand side with the observation in \eqref{eq:F_em_x_lb_1} and the right hand side with the bound in the proof of Proposition \ref{pr:Femconvergencetheta},
	\begin{align*}
		F(\theta_{k+1},q_{k+1}) - F_\star \leq e^{-\lambda(h_{k+1}+1/L_\theta)}(F(\theta_k,q_k)-F_\star) + h^2_{k+1}B.
	\end{align*}
	The result follows upon interating this inequality, and taking the infimum across sequences $\{h_k\}$ such that $h_k\leq 1/4L_x$.
\end{proof}

As before, convergence in $\mathsf{d}$-distance of both the sequence of parameters and the posterior now follows by the extension of the Talagrand inequality via Theorem \ref{thm:exttalagrand}.
\begin{corollary} \label{cor:emconvergence}
	Under the same conditions as Theorem \ref{thm:Femconvergence}, 
	\begin{align*} \label{eq:emconvergence}
		\lambda \mathsf{d}((\theta_k,q_k),\cal{M}_\star)^2 &\leq 2\inf\bigg\{e^{-\lambda(k/L_\theta + \sum_{j=1}^{k} h_j)}(F(\theta_0,q_0)-F_\star) + B\sum_{j=1}^{k}h^2_{j} e^{-\lambda(k/L_\theta+ \sum_{i=j+1}^{k} h_i)} \bigg\} 
	\end{align*}
\end{corollary} 
If only Assumption \ref{ass:gradLip}(i) holds, the bound of Corollary \ref{cor:emconvergencetheta} still does.
Since $\cal{M}_\star=\{(\theta_\star,\pi_{\theta_\star}):\theta_\star \in \cal{O}_\star\}$, we are bounding the distance of the EM iterates to a (local) maximum of the marginal likelihood and the corresponding posterior, which can be identified as the projection of $(\theta_k,q_k)$ onto the optimal set $\cal{M}_\star$.
By definition, the log-Sobolev constant $\lambda$, which dictates the convergence rate of EM, has the interpretation of being the maximum ratio between information and the free energy produced along the EM iterates, as given by the functionals $I$ \eqref{eq:fisherinformation} and $F$ \eqref{eq:entropy}, suggesting some analogies with the ratio of missing information of classical asymptotic theory \citep{Mclachlan2007}.  

The bound above is reminiscent of the convergence bounds in Langevin Monte Carlo when $\{h_k\}$ are its discretization step sizes, and is obtained with the observation \eqref{eq:F_em_x_lb_1} that an E-step is at least as good as the best possible Langevin step to decrease free energy. Therefore, its step sizes become a parameter we can optimize over. Notice that, as opposed to classical Langevin Monte Carlo bounds \citep{Durmus2019}, the rightmost term in Theorem \ref{thm:Femconvergence}'s bound is \textit{not} a bias term, as we can always just take $\{h_k\}=0$ and eliminate it altogether, while retaining a non-zero convergence rate.

It does not appear immediate to quantify how much sharper Theorem \ref{thm:Femconvergence} is than Proposition \ref{pr:Femconvergencetheta} asymptotically. From a non-asymptotic error bounds perspective, the difference seem to be substantial, and we refer to Appendix \ref{sec:illustration} for a simple illustration.

Appendix~\ref{sec:logsobsufficient} provides a mechanism for verifying the extended log Sobolev condition with applications to Bayesian hierarchical  models, and Appendix~\ref{sec:preservation} some results which can be used to extend its domain of applicability even to some non-concave settings: we know, for instance, that the same bound holds if we replace the extended log-Sobolev inequality assumption with $\ell$'s strong concavity, or any bounded perturbation, in the sense of Proposition \ref{pr:perturbationprinciple}, of such model. 	
When $\ell$ is strongly concave, $\ell$ has a unique maximizer, as does the marginal likelihood $Z_\theta$, and $\cal{M}_\star$ is a singleton. At least without additional assumptions, we are required to start the algorithm with an initial distribution $q_0$ having a density, otherwise the term $F(\theta_0,q_0)$ would be infinite. In practice, this is never an issue, and we can avoid this problem altogether by performing an additional E-step as part of the initialization procedure.

We now show that we can use similar techniques to analyze alternatives to the EM algorithm that we consider when either, or both, the E-M steps are intractable.

\subsection{First-order EM}

Often, the M-step is intractable. In these cases, rather than solving the related maximization problem exactly, one often performs instead a gradient step to find a $\theta$ with a smaller, but not optimal, value of $\theta\mapsto F(\theta,q_k) \propto - \int \ell(\theta,x)q_k(d x)$. This results in the first-order EM (or gradient EM) algorithm (see, e.g.~\citet{Balakrishnan2017,Yan2017} and also \citet{Lange1995} for examples where the M-step is intractable, but gradient-based strategies can be implemented).

\begin{algo} \label{alg:gradem}
	First-order EM. Inputs: initial values $(\theta_0,q_0)$, step sizes $\{h_k\}$ 
	\begin{tabbing}
		\qquad \enspace For $k\geq 0$ \\
		\qquad \qquad  Update the parameter estimate $\theta_{k+1} = \theta_k + h_{k+1}\int \nabla_\theta \ell(\theta_k,x)q_{k}(d x)$ \\
		\qquad \qquad Update the posterior estimate $	q_{k+1} = \pi_{\theta_{k+1}}$ 
	\end{tabbing}
\end{algo}

We can easily use the tools developed above to study first-order EM. Indeed, provided $h\leq 1/L_\theta$, we can verbatim adapt the proofs of Lemma \ref{lemma:Fdecr}, Proposition \ref{pr:Femconvergencetheta} and Corollary \ref{cor:emconvergencetheta} with $h$ in place of $1/L_\theta$, yielding:
\begin{theorem} \label{thm:grademconv}
	Let Assumptions \ref{ass:model} and \ref{ass:gradLip}(i) hold, and assume that the measures $(\rho_\theta(d x))_{\theta\in\r^{d_\theta}}$ satisfy the extended log-Sobolev inequality with constant $\lambda>0$. Then, if $h_j\leq 1/L_\theta$ for all $j\leq k$,
	\begin{equation*}
		\lambda \mathsf{d}((\theta_k,q_k),\cal{M}_\star)^2 \leq 2e^{-\lambda \sum_{j=1}^k h_j}(F(\theta_0,q_0)-F_\star).
	\end{equation*}
\end{theorem}
Because the parameter update here is not a minimization step anymore, we cannot leverage the posterior updates to improve the convergence bound in the same way as for the basic EM algorithm. Therefore, we can ensure first-order EM convergence with a small enough choice of the step size, but the error bounds for first-order EM algorithm will in general decrease more slowly than those for EM, which of course is consistent with the way in which we would expect the actual error to behave. 

\subsection{Langevin EM}

In a complementary setting, the E-step might be intractable. Hypothetically, in these cases, we could consider following the same strategy as the first-order EM algorithm, although it is perhaps less obvious how to implement such a scheme: rather than performing the E-step, which is minimizing the free energy for a fixed $\theta_{k+1}$ (Proposition \ref{pr:emiscd}), one can take a gradient step in the space of probability measures to find a new $q_{k+1}$ which reduces $q\mapsto F(\theta_{k+1},q) \propto \KL(q||\pi_{\theta_{k+1}})$. Given the previous iterate $q_k$, for a step size $h>0$, as we argued in Section \ref{sec:em}, in the Wasserstein-2 geometry this step consists in $q_k\rightarrow \textup{Law}(X_{k+1})$, where
\begin{equation} \label{eq:pgdestep}
	X_{k+1} = X_k + h\nabla_x \ell(\theta_{k+1},X_k) + \sqrt{2h}\xi_k, \quad \textup{Law}(X_k)=q_k,
\end{equation}
and where $\xi_k$ is a standard normal random variable. We term the resulting algorithm Langevin EM. Although this algorithm may appear somewhat contrived, we view it as an idealisation of a (Markov chain) Monte Carlo EM \citep{Wei1990} in which $q_k$ is approximated with a particle system and the particles within the approximation are updated by an application of an unadjusted Langevin kernel of step-size $h$ and invariant distribution $\pi_{\theta_k}$ at each step \citep{Roberts1996}. In general, the use of Markov chain Monte Carlo steps within EM algorithms has been explored in a number of contexts, see e.g. \citet{Mcculloch1994,Levine2001}. In particular, this implementation was introduced and studied in \citet[Appendices D and F]{PGD} as the gradient flow approximated by ``marginal particle gradient descent''.

\begin{algo} \label{alg:langem}
	Langevin EM. Inputs: initial values $(\theta_0,q_0)$, step sizes $\{h_k\}$
	\begin{tabbing}
		\qquad \enspace For $k\geq 0$ \\
		\qquad \qquad   Update the parameter estimate $\theta_{k+1} = \argmax_\theta \int \ell(\theta;x)q_{k}(dx)$ \\
		\qquad \qquad Update the posterior estimate $q_{k+1}=\textup{Law}(X_{k+1})$, where\\ \qquad \qquad $	X_{k+1}= X_k + h_{k+1}\nabla_x \ell(\theta_{k+1},X_k) + \sqrt{2h}\xi_k$ 
	\end{tabbing}
\end{algo}

We can study Langevin EM almost immediately with our results: we follow verbatim the proof of Theorem \ref{thm:Femconvergence} until \eqref{eq:F_em_x_lb_3}. Then, rather than combining the bound with the one in Proposition \ref{pr:Femconvergencetheta}, which we cannot use as the posterior update is not a minimization step here, we use the looser bound $F(\theta_{k+1},q_k)\leq F(\theta_k,q_k)$, iterate the resulting inequality and then combine it with Theorem \ref{thm:exttalagrand} as usual. 
\begin{theorem} \label{thm:langemconvergence}
	Let Assumptions \ref{ass:model} and \ref{ass:gradLip}(ii) hold, and assume that the measures $(\rho_\theta(d x))_{\theta\in\r^{d_\theta}}$ satisfy the extended log-Sobolev inequality with constant $\lambda>0$. Then, if $h_j\leq 1/4L_x$ for all $j\leq k$,
	\begin{equation*} 
		\lambda \mathsf{d}((\theta_k,q_k),\cal{M}_\star)^2 \leq 2e^{-\lambda \sum_{j=1}^k h_j}(F(\theta_0,q_0)-F_\star) +  B\sum_{j=1}^{k}h^2_{j} e^{-\lambda \sum_{i=j+1}^{k} h_i}. 
	\end{equation*}
\end{theorem}
$B$ is as defined as in Theorem \ref{thm:Femconvergence}. In contrast to the EM case, we have not shown that $B$ is always finite here. However, doing so is not typically onerous and is, for example, immediately true  if the log-likelihood is strongly concave in the tails uniformly on $\theta$, as can be shown by adapting e.g.~the arguments of Corollary 2.3 in \citet{Cattiaux2008}.

Unlike EM, or first-order EM, Langevin EM is biased: if we keep constant step sizes, there is no way to cancel the rightmost term without getting a null rate of convergence. However, we can control the bias by choosing a small step size, or by decreasing these appropriately. This is because the forward discretization of the Langevin diffusion \eqref{eq:pgdestep} introduces a bias, and is consistent with results for Langevin Monte Carlo \citep{Wibisono2018,Durmus2019}. Similar considerations to those of first-order EM apply here.

\subsection{Gradient Descents} \label{sec:cwpgd}
When both E- and M- steps are intractable, we can take a gradient step in both directions at each iteration; essentially replacing the usual alternating minimization optimization of the free energy with an alternating (or coordinate-wise) gradient descent method, as one might in Euclidean optimization problems for which both the coordinate-wise minimization problems are not tractable. This results in the Alternating Gradient Descent Algorithm \ref{alg:cwpgd} below, a mean-field limit version of the algorithm of \citet{Bortoli2021} when the batch sizes are set to one, and when the parameter updates can be carried out explicitly. 

\begin{algo} \label{alg:cwpgd}
	Alternating Gradient Descent. Inputs: initial values $(\theta_0,q_0)$, step sizes $\{h_k\}$
	\begin{tabbing}
		\qquad \enspace For $k\geq 0$ \\
		\qquad \qquad  Update the parameter estimate $\theta_{k+1} = \theta_k + h_{k+1}\int \nabla_\theta \ell(\theta_k;x)q_{k}(d x)$ \\
		\qquad \qquad Update the posterior estimate $q_{k+1}=\textup{Law}(X_{k+1})$, where \\ \qquad \qquad $	X_{k+1}= X_k + h_{k+1}\nabla_x \ell(\theta_{k+1},X_k) + \sqrt{2h}\xi_k$ 
	\end{tabbing}
\end{algo}

\begin{theorem} \label{thm:cwpgdconvergence}
	Let Assumptions \ref{ass:model}--\ref{ass:gradLip} hold, and assume that the measures $(\rho_\theta(d x))_{\theta\in\r^{d_\theta}}$ satisfy the extended log-Sobolev inequality with constant $\lambda>0$. Let $L:=\max(L_\theta,L_x)$. Then, if $h\leq 1/8L$, for all $k\in\n$,
	\begin{equation*} \label{eq:Fconvergence}
		\lambda \mathsf{d}((\theta_k,q_k),\cal{M}_\star)^2 \leq 2e^{-\lambda \sum_{j=1}^k h_j}(F(\theta_0,q_0)-F_\star) + 24L^2d_x \sum_{j=1}^{k}h^2_{j} e^{-\lambda \sum_{i=j+1}^{k} h_i}
	\end{equation*}
\end{theorem}

It is also possible to update both parameters and distributions simultaneously (rather than coordinate-wise, as above), in which case we would obtain (a mean-field limit version of) the particle gradient descent algorithm proposed in \cite{PGD}, studied and extended in \citet{Lim2023,Caprio2025}, or of the interacting Langevin algorithm \citep{Akyildiz2023}.

\begin{algo} \label{alg:pgd}
	Gradient Descent. Inputs: initial values $(\theta_0,q_0)$, step sizes $\{h_k\}$
	\begin{tabbing}
		\qquad \enspace For $k\geq 0$ \\
		\qquad \qquad  Update the parameter estimate $\theta_{k+1} = \theta_k + h_{k+1}\int \nabla_\theta \ell(\theta_k;x)q_{k}(d x)$ \\
		\qquad \qquad Update the posterior estimate $q_{k+1}=\textup{Law}(X_{k+1})$, where \\ \qquad \qquad $	X_{k+1}= X_k + h_{k+1}\nabla_x \ell(\theta_{k},X_k) + \sqrt{2h}\xi_k$ 
	\end{tabbing}
\end{algo}

Algorithm \ref{alg:pgd} is the discretization of the gradient flow of the free energy w.r.t.~the Euclidean--Wasserstein metric on $\cal{M}_2$ \citep{PGD,Caprio2025}, see also Appendix \ref{app:geom} for further details.

\begin{theorem} \label{thm:pgdconvergence}
	Let Assumptions \ref{ass:model}--\ref{ass:gradLip} hold, and assume that the measures $(\rho_\theta(d x))_{\theta\in\r^{d_\theta}}$ satisfy the extended log-Sobolev inequality with constant $\lambda>0$. Let $L:=\max(L_\theta,L_x)$. Then, if $h\leq 1/4L$, for all $k\in\n$,
	\begin{equation*} 
		\lambda \mathsf{d}((\theta_k,q_k),\cal{M}_\star)^2 \leq 2e^{-\lambda \sum_{j=1}^k h_j}(F(\theta_0,q_0)-F_\star) + 12L^2d_x \sum_{j=1}^{k}h^2_{j} e^{-\lambda \sum_{i=j+1}^{k} h_i}
	\end{equation*}
\end{theorem}
Like Langevin EM, both $F$'s gradient descents algorithms are expectedly biased if step sizes are constant, and this is in line with previous analyses of similar algorithms \citep{Caprio2025,Akyildiz2023,Encinar2024,Oliva2024}. 

A comparison of the results for EM, first-order EM, Langevin EM and both the gradient descents approaches shows that the vanilla EM algorithm is the fastest of the five in terms of number of iterations (at least for this class of models). While comparing error \emph{bounds} cannot provide definitive statements, this is consistent with the way one would expect the errors themselves to behave in general, because EM is solving two entire minimization problems at each iterations, whereas first-order EM and Langevin EM are solving one and taking a gradient step for the other, and the gradient descents approaches are actually only taking a gradient steps both directions. Our bounds also suggest that the gradient descent's performances might be comparable to that of the alternating gradient descent algorithm, but that in the former one might take larger step sizes. The intuition might be that, because in Algorithm \ref{alg:cwpgd} we take the same step size for both coordinates, and we take gradient steps coordinate-wise, we have to be more conservative in the steps as to not overshoot in each of these. We speculate that different step sizes, depending on the interplay between $L_x$ and $L_\theta$, might also lead to sharper step sizes guarantees for the alternating approach in some cases. 

This hierarchy between alternating minimization and gradient descent type algorithms is consistent with the corresponding results and behaviour in Euclidean space \citep{Beck2013,Spall2012}. 
These results do not necessarily imply that implementable variants of EM algorithm outperforms the implementable form of any of these alternatives in practice, although it will certainly be the case when the minimization steps can be implemented exactly. In Appendix \ref{sec:illustration} we present a simple example showing performance and bounds for each of the algorithms under examination, which corroborate these findings.
Algorithms \ref{alg:em}--\ref{alg:pgd} are idealized algorithms, and for many other models of interest, particularly in large-scale modern applications, the iterations therein are still intractable. In practice, one often resorts to Monte Carlo approximations. As noted above, for EM and first-order EM, it is common to substitute the E-steps with $N$ Monte Carlo samples from $\pi_{\theta_{k+1}}$ (the Monte Carlo EM algorithm \citep{Wei1990}) whereas one would make Algorithms \ref{alg:langem}--\ref{alg:pgd} implementable by approximating $\textup{Law}(X_{k+1})$ with the empirical distribution of $N$ particles following the specified Langevin dynamics.  An analysis of these Monte Carlo versions of the algorithms under the extended log-Sobolev inequality and using the techniques we introduced here is left for future work.

\section{Discussion} \label{sec:discussion}

This paper established non-asymptotic error bounds and convergence rates for the EM algorithm, and some of its variants, under a log-Sobolev type inequality. 
The log-Sobolev constant, which dictates the convergence rate in EM (Corollary \ref{cor:emconvergence}), has the natural interpretation of being bounded by the maximum ratio between information and free energy produced along the EM iterates (as measured by the functionals $I$ and $F$). The convergence bounds seem to be state-of-the-art in terms of characterizing the exponential regime of EM, at least when the latent space is continuous. 
Starting from the observation of \citet{Neal1998} that EM corresponds to a coordinate-wise minimization procedure on the product of Euclidean and the space of probability distributions, this approach can be considered a generalization of standard arguments used to study alternating minimization on Euclidean space \citep{Beck2013} via concepts in optimal transport.

We believe the main contribution of this paper is to demonstrate that generalisation of methods from Euclidean optimization to the Euclidean--Wasserstein product space underlying EM provides a natural route to characterizing the EM algorithm theoretically. The perspective established herein, and the simplicity of the proofs, demonstrate that there is considerable potential to further leverage the literatures of optimal transport and functional inequalities to better understand the EM algorithm. Below we mention some limitations of the work, as well as some potential
connections and extensions, some of which form part of current or future research. 

To obtain non-asymptotic bounds, our strategy was to compare the E-M-steps with appropriate gradient-based algorithms, for which we were able to quantify the free energy's decrease using the model's assumed smoothness. On the other hand, EM \textit{always} decreases the free energy, regardless of smoothness. This suggests that, for non-smooth models, one might compare the E-M-steps with different algorithms that are more appropriate in these settings (for instance, proximal methods \citep{Salim2020}).  On a similar note, we remark that the EM iterations make no use of any underlying metric on $\cal{M}$, and we could have considered others, and thus different induced log-Sobolev type inequalities. For instance, by considering the gradient of the free energy $\textup{grad}_{\cal{M}_2}F$ induced by the product of Euclidean and Stein geometry on $\cal{P}_2(\r^{d_x})$ \citep{Liu2016,Duncan2023}. These would then require comparing the E-M steps to different algorithms, and they might translate into different practical conditions on the underlying model (as analogues of the results in Appendix \ref{sec:diffinequalities}), and could better characterize EM's convergence in some settings.

Appendix~\ref{sec:diffinequalities} provides generalizations of several standard results (the Bakry--{\'E}mery criterion, the Holley--Stroock and the contraction principles) to allow verification of the extended log-Sobolev inequality and comparison of the performance of models with different completions, and we expect that many other results of this flavour can be adapted from the log Sobolev inequality literature, further characterizing the exponential convergence regime of EM. 
Regardless, global exponential convergence of EM is not typical in applications, and the extended log-Sobolev inequality is shown to hold here only for particularly well-behaved Bayesian hierarchical models. However, exploiting the connections we established, we anticipate that one might characterize slower-than-exponential and local convergence regimes for wider classes of models by leveraging the connections with functional inequalities. For instance, we anticipate that it is possible to obtain `weak' or `modified' versions of the extended log-Sobolev inequality, as done in \citep{Toscani2000,Rockner2001,Cattiaux2007,Andrieu2022} for more standard functional inequalities. For instance, by considering inequalities of the form
\begin{equation*}
	2 \lambda (F (\theta, q) - F_\star)^\alpha \leq  \norm{\textup{grad}_{\cal{M}_2}F(\theta,q)}_{\cal{M}_2}^2
\end{equation*}
for some $\alpha> 1$, which, used within Proposition \ref{pr:Femconvergencetheta}, would lead to
\begin{equation*}
	F(\theta_{k},q_{k})-F(\theta_{k+1},q_{k+1}) \geq (\lambda/L_\theta)(F(\theta_k,q_k)-F_\star)^\alpha
\end{equation*}
i.e. to a \emph{sub-geometric} decrease of free energy along EM iterates.

A limitation of the approach is the focus on continuous Euclidean state spaces. While these constitute an important class of the models to which EM is applied, in particular within the Empirical Bayes framework, many other typical applications, such as Gaussian Mixture Models, involve discrete latent spaces. We believe that it should be possible to obtain analogues of our results in the discrete setting by leveraging advances on the study of Wasserstein gradients on discrete spaces \citep{Chow2018}.  Similarly, within a continuous state spaces setting, it should be possible to consider appropriate Riemannian manifolds, rather than $\r^{d_x}$.

\section*{Acknowledgement}
The authors thank Juan Kuntz for very helpful comments on an earlier version of this manuscript, and the reviewers for their very constructive feedback. RC was funded by the UK Engineering and Physical Sciences Research Council (EPSRC) via studentship 2585619 as part of grant number EP/W523793/1. AMJ acknowledges further support from EPSRC under grant numbers EP/R034710/1 and from United Kingdom Research and Innovation (UKRI) via grant number EP/Y014650/1, as part of the ERC Synergy project OCEAN. 

\printbibliography
\appendix
\appendixpage
\section{Notation}
	\begin{tabular}{ c l }
		$\rho_\theta(x)$ & Joint distribution $p_\theta(x,y)$ of latent variables $x$ and data $y$  \\
		$\ell(\theta,x)$ & Logarithm of $\rho_\theta(x)$  \\
		$Z_\theta$ & The marginal likelihood, $Z_\theta:=\int \rho_\theta(x)d x$  \\ 
		$Z_\star$ & Value of $Z_\theta$ at a stationary point \\     
		$\pi_\theta(x)$ & Posterior of $x$ given $y$, $\pi_\theta(x)=p_\theta(x|y)=\rho_\theta(x)/Z_\theta$  \\
		$F(\theta,q)$ & Free energy functional \eqref{eq:entropy} \\
		$I(\theta,q)$ & Extended Fisher information functional \eqref{eq:fisherinformation} \\
		$I(q||\pi_\theta)$ & Relative Fisher information functional \eqref{eq:stdfisherinformation} \\ 
		$\cal{O}_\star$ & Local maxima of the marginal likelihood \\
		$\cal{P}_2(\r^{d_x})$ & Space of probability measures with densities and finite second moments \\
		$\cal{M}_2$ & Product spaces, $\cal{M}_2:=\r^{d_\theta}\times \cal{P}_2(\r^{d_x})$ \\
		$\cal{M}_\star$ & $F$'s optimal set in $\cal{M}$, $\cal{M}_\star=\argmin F = \{(\theta_\star,\pi_{\theta_\star}):\theta_\star \in \cal{O}_\star\}$ \\ 
		$\iprod{\cdot}{\cdot}$ and $\norm{\cdot}$ & Euclidean inner product and norm \\
		$\iprod{\cdot}{\cdot}_{\cal{L}_2(d x)}$ & $\cal{L}_2(\r^{d_x},d x)$'s inner product \\
		$\iprod{\cdot}{\cdot}_{W_2}, \iprod{\cdot}{\cdot}_{\cal{M}_2}$ & Wasserstein-2 and $\cal{M}_2$ inner products (Definitions \ref{def:W2iprod} and \ref{def:M2iprod}) \\
		$\mathsf{d}_E,\mathsf{d}_{W_2}$ and $\mathsf{d}$ & Euclidean distance, Wasserstein-2 distance, and their product \eqref{eq:ddeterministic} \\
		$\nabla$ and $\nabla\cdot$ & Gradient and divergence operators \\
		$\Delta$ & Laplacian operator, $\Delta:=\nabla\cdot\nabla$
	\end{tabular}

\section{Riemannian structure of $\cal{M}_2$} \label{app:geom}

In this section, we give some details on the Riemannian structure of $\cal{M}_2$,  we  show how we can derive the gradient of the free energy in $\cal{M}_2$'s geometry, which is key for our analysis, and which we need to make sense of the inequality \eqref{eq:extlogsobolev}. In particular, we derive an expression for $\textup{grad}_{\cal{M}_2}F(\theta,q)$ and its norm, motivating the use of the functional \eqref{eq:fisherinformation} in the manuscript, and also Algorithm \ref{alg:pgd}.
To do so, we leverage the formal interpretation of $\cal{P}_2(\r^{d_x})$ endowed with the Wasserstein-2 distance $\mathsf{d}_{W_2}$ as a Riemannian-manifold, an approach pioneered by \citet{Otto2001}. These computations were carried out in \citet{PGD} to make sense of gradient descent for $F$, which we discuss in Section \ref{sec:cwpgd} of the manuscript. In this section we skim over technical details; this approach is not completely rigorous but more technical arguments following the approach in \cite{ambrosio2005} yield the same result. 
None of our results depend upon this informal reasoning. 

If we think of $\cal{M}_2$ as a Riemannian manifold, in order to define gradients on $\cal{M}_2$, we need a sensible notion of tangent space and inner product. Because of the product structure of $\cal{M}_2=\r^{d_\theta}\times \cal{P}_2(\r^{d_x})$, our tangent space will be the product of the tangent spaces of each component. For the  $\r^{d_\theta}$ component, we identify the tangent space at any point $\theta$ as  $\r^{d_\theta}$ itself, so that $\cal{T}_\theta \r^{d_\theta}=\r^{d_\theta}$. For $\cal{P}_2(\r^{d_x})$ we consider the geometry induced by the Wasserstein-2 distance $\mathsf{d}_{W_2}$. With this choice, we think at the tangent space at any point $q\in \cal{P}_2(\r^{d_x})$ as the space of Lebesgue-integrable functions with zero integral $	\cal{T}_q \cal{P}_2(\r^{d_x}):=\left\{ h: \int h(x)d x = 0 \right\}$,
which we endow with the following inner product.
\begin{definition}[Wasserstein-2 inner product] \label{def:W2iprod}
	Given two elements $h_1,h_2\in\cal{T}_q \cal{P}_2(\r^{d_x})$, we define their Wasserstein-2 inner product at $q\in\cal{P}_2(\r^{d_x})$ as  
	\begin{equation*}
		\iprod{h_1}{h_2}_{W_2}:=\int \iprod{\nabla_x \psi_1(x)}{\nabla_x \psi_2(x)}q(d x), \quad \text{where $\psi_i$ solves} \quad \nabla_x \cdot(q\nabla_x \psi_i) = -h_i \quad i=1,2.
	\end{equation*}
	and we denote with $\norm{\cdot}_{W_2}$ the induced norm.
\end{definition}
One motivation for this choice of tangent space and inner product comes from comparing the Benamou--Brenier formula \citep[p.159]{Villani2009},
\begin{equation} \label{eq:benamoubrenierformula}
	\mathsf{d}_{W_2}(q,p)= \inf_{q_t} \left\{ \int_0^1 \inf_{v_t} \left\{ \int \norm{v_t}^2 q_t(d x)d t: \nabla_x\cdot(q_t v_t)=-\partial_t q_t \right\}: q_0=q,  q_1=p \right\},
\end{equation}
with the formula for the distance between two points in a Riemannian manifold $(M,\mathsf{d}_M)$ 
\begin{equation*}
	\mathsf{d}_M(p,q)=\inf_{q_t} \left\{\int_0^1 \norm{\partial_t q_t}_M d t :  q_0=q,  q_1=p \right\},
\end{equation*}
upon noting that the optimal `velocity' field $v_t$ realizing the infimum in \eqref{eq:benamoubrenierformula} is achieved by a gradient of a function $\nabla_x \psi$, and that for $\nabla_x\cdot(q\nabla_x\psi)=-h$ to be solvable it is required that $h$ has zero Lebesgue integral. See \citet{Figalli2021,Villani2009} for more details. 

For a point $(\theta,q)\in\cal{M}_2$, we set $\cal{T}_{(\theta,q)}\cal{M}_2:=\cal{T}_\theta\r^{d_\theta}\times \cal{T}_q\cal{P}_2(\r^{d_x})$ and we endow $\cal{M}_2$ with the following inner product which arises naturally from the product-space structure.

\begin{definition}[$\cal{M}_2$ inner product] \label{def:M2iprod}
	Given two elements $(a_1,h_1),(a_2,h_2)\in \cal{T}_{(\theta,q)}\cal{M}_2$, we define their $\cal{M}_2$ inner product as
	\begin{equation*}
		\iprod{(a_1,h_1)}{(a_2,h_2)}_{\cal{M}_2}:=\iprod{a_1}{a_2} + 	\iprod{h_1}{h_2}_{W_2}
	\end{equation*}
	and we denote with $\norm{\cdot}_{\cal{M}_2}$ the induced norm.
\end{definition}

We can now define the gradient on  $\cal{M}_2$ in analogy with Riemannian geometry (see, e.g., Section 3 in \cite{Boumal2023}).

\begin{definition}[Gradients in $\cal{M}_2$] \label{def:M2grad}
	For a functional $F$ on $\cal{M}_2$, its $\cal{M}_2$-gradient at $(\theta,q)\in\cal{M}_2$ is the unique function $\textup{grad}_{\cal{M}_2}F(\theta,q)=(\textup{grad}_{\r^{d_\theta}}F(\theta,q),\textup{grad}_{\cal{P}_2(\r^{d_x})}F(\theta,q))$ such that
	\begin{equation*}
		\frac{d}{d t}\bigg\vert_{t=0} F(\theta_t,q_t) = \iprod{\textup{grad}_{\cal{M}_2}F(\theta,q)}{\left(\frac{d}{d t} \bigg\vert_{t=0} \theta_t , \frac{d }{d t}\bigg\vert_{t=0}q_t \right)}_{\cal{M}_2}
	\end{equation*}
	for any smooth curve $t\mapsto (\theta_t,q_t)$ such that $(\theta_0,q_0)=(\theta,q)$, provided that it exists.
\end{definition}	

The first variation will prove useful in calculating gradients, as the following lemma shows.
\begin{definition}[First variation] \label{def:firstvar}
	For a functional $F$ on $\cal{M}_2$, its first variation at $(\theta,q)\in\cal{M}_2$ is the unique (up to an additive constant) function $\delta_{\cal{M}_2} F(\theta,q) = (\delta_{\r^{d_\theta}} F(\theta,q), \delta_{\cal{P}_2(\r^{d_x})} F(\theta,q))$ such that
	\begin{equation*}
		\frac{d}{d t}\bigg\vert_{t=0} F(\theta_t,q_t) = \iprod{\delta_{\cal{M}_2} F(\theta,q)}{\left(\frac{d}{d t} \bigg\vert_{t=0} \theta_t , \frac{d }{d t}\bigg\vert_{t=0}q_t \right)}_{\cal{L}^2(d x)}
	\end{equation*}
	for any smooth curve $t\mapsto (\theta_t,q_t)$ such that $(\theta_0,q_0)=(\theta,q)$, provided that it exists.
\end{definition}
\begin{lemma} If it exists, \label{lemma:gradcomp} $\textup{grad}_{\cal{M}_2}F(\theta,q)$ satisfies 
	\begin{equation*}
		\textup{grad}_{\cal{M}_2}F(\theta,q) = \left(\delta_{\r^{d_\theta}} F(\theta,q),-\nabla_x \cdot(q\nabla_x \delta_{\cal{P}_2(\r^{d_x})} F(\theta,q))\right).
	\end{equation*}
	
\end{lemma}
\begin{proof}
By comparing Definitions \ref{def:M2grad} and \ref{def:firstvar} we observe
\begin{align*}
	\frac{d}{d t}\bigg\vert_{t=0} F(\theta_t,q_t) 
	&= \iprod{ \delta_{\r^{d_\theta}} F(\theta,q)}{\frac{d}{d t} \bigg\vert_{t=0} \theta_t} + \iprod{\delta_{\cal{P}_2(\r^{d_x})} F(\theta,q)}{\frac{d }{d t}\bigg\vert_{t=0}q_t}_{\cal{L}^2(d x)} \\
	&= \iprod{\textup{grad}_{\r^{d_\theta}}F(\theta,q)}{\frac{d}{d t} \bigg\vert_{t=0} \theta_t} + \iprod{\textup{grad}_{\cal{P}_2(\r^{d_x})}F(\theta,q)}{\frac{d }{d t}\bigg\vert_{t=0}q_t}_{W_2}.
\end{align*}
In order to proceed, we refer to Defininition~\ref{def:W2iprod} and allow $\psi$ to denote the solution of $\nabla_x\cdot(q\nabla_x\psi)=-(d q_t/d t)|_{t=0}$.  The first terms on the right hand side clearly coincide. We then have, upon equating the second terms on the  right hand side of the above display and inserting $\psi$ so defined:
\begin{align*}
	\iprod{\textup{grad}_{\cal{P}_2(\r^{d_x})}F(\theta,q)}{\frac{d }{d t}\bigg\vert_{t=0} q_t}_{W_2} &= - \int \delta_{\cal{P}_2(\r^{d_x})} F(\theta,q)(x) \nabla_x \cdot (q(x)\nabla_x \psi(x))d x \\
	&= \int \iprod{\nabla_x \delta_{\cal{P}_2(\r^{d_x})} F(\theta,q)(x)}{\nabla_x \psi(x)}q(d x)
\end{align*}
where we integrated by parts. We conclude by noting that by definition of the Wasserstein innner product, Definition~\ref{def:W2iprod}, we know that $\textup{grad}_{\cal{P}_2(\r^{d_x})}F(\theta,q)$ has to satisfy $\nabla_x \cdot(q \nabla_x  \delta_{\cal{P}_2(\r^{d_x})} F(\theta,q)) = - \textup{grad}_{\cal{P}_2(\r^{d_x})}F(\theta,q)$ and the result is immediate.
\end{proof}

Using Definition \ref{def:firstvar} we readily compute $\delta_{\cal{M}_2} F(\theta,q)=(-\int \nabla_\theta \ell(\theta,x)q(d x),\log(q/\rho_\theta))$, and using the above lemma we identify the following gradient of the free energy:
\begin{equation} \label{eq:Fgrad}
	\textup{grad}_{\cal{M}_2}F(\theta,q) = \left(-\int \nabla_\theta \ell(\theta,x)q(d x) , -\nabla_x\cdot \left\{ q(\nabla_x \log(q)-\nabla_x\log(\rho_\theta)\right\} \right).
\end{equation}
The computation of the squared norm of the gradient of the free energy in $\cal{M}_2$ given in \eqref{eq:fisherinformation} follows. Futhermore, we can also make sense of $F$'s gradient flow in $\cal{M}_2$, which justifies Algorithm \ref{alg:pgd}: using, \eqref{eq:Fgrad}, we readily observe that it is given by
\begin{equation*}
	\partial_t \theta_t = \int \nabla_\theta \ell(\theta_t,x)q_t(d x), \quad \partial_t q_t = \nabla_x\cdot \left\{ q_t(\nabla_x \log(q_t)-\nabla_x\log(\rho_{\theta_t})\right\}
\end{equation*}
and we notice that Algorithm \ref{alg:pgd}'s iterates are just the discretization of these evolution equations.

Finally, we are also in a good position to motivate the use of the Langevin steps as a gradient step in $\cal{P}_2(\r^{d_x})$ \citep{Otto2001}, leading to the proof of Theorem \ref{thm:Femconvergence}, since such space is marginal in $\cal{M}_2$.  In particular,
when the parameter space is the trivial space $\{\theta_{k+1}\}$ from \eqref{eq:Fgrad} we have
\begin{equation*}
	\textup{grad}_{\cal{P}_2(\r^{d_x})}\KL(q||\pi_{\theta_{k+1}}) = \nabla_x\cdot \left\{q(\nabla_x\log(q)-\nabla_x\log(\pi_{\theta_{k+1}}))\right\}
\end{equation*}
which implies that the curve $t\mapsto q_t$ defined by 
\begin{equation*}
	\partial_t q_t = \textup{grad}_{\cal{P}_2(\r^{d_x})}\KL(q_t||\pi_{\theta_{k+1}}) = \nabla_x\cdot \left\{q_t(\nabla_x\log(q_t)-\nabla_x\log(\pi_{\theta_{k+1}}))\right\}
\end{equation*} 
is the Wasserstein gradient flow of $q\mapsto \KL(q||\pi_{\theta_{k+1}})$ in $\cal{P}_2(\r^{d_x})$ (the steepest descent curve on the space of probability distributions that connects an initial distribution to the invariant $\pi_{\theta_{k+1}}$). This differential equation is the Fokker--Planck equation of the overdamped Langevin diffusion
\begin{equation*}
	d X_t = \nabla_x \log(\pi_{\theta_{k+1}}(X_t))d t + \sqrt{2}d W_t = \nabla_x \ell(\theta_{k+1},X_t)d t + \sqrt{2}d W_t
\end{equation*}
where $W_t$ denotes a Brownian Motion. For more on Wasserstein gradient flows, see e.g. \citet{Wibisono2018,Figalli2021,Chewi2023} and references therein.

\section{On the extended log-Sobolev inequality} \label{sec:diffinequalities}

In this Appendix, we study sufficient conditions on the model to verify the extended log-Sobolev inequality \eqref{eq:extlogsobolev}, present some examples of Hierarchical models where it holds, showing that \eqref{eq:extlogsobolev} always hold in strongly log-concave models, and that it may even go beyond concavity.

\subsection{Bakry--{\'E}mery and strongly log-concave models} \label{sec:logsobsufficient}

A generalization of the \citet{Bakry1985} criterion shows that a convenient sufficient (but not necessary) condition for the extended log-Sobolev inequality is strong log-concavity.
\begin{assumption}[Strong log-concavity]\label{ass:strongconcave} There exists a $\lambda>0$ such that 
	\begin{equation*} \label{eq:ellstrongconcave}
		\ell ( (1-t) \theta + t \theta', (1-t) x + t x') \geq (1 - t) \ell (\theta,x) + t \ell (\theta',x') + \frac{\lambda t(1-t)}{2} \norm{(\theta,x) - (\theta',x')}^2,
	\end{equation*}   
	for all $(\theta,x), (\theta',x')$ in $\mathbb{R}^{d_\theta} \times \mathbb{R}^{d_x}$ and $0 \leq t \leq 1$.
\end{assumption}
Assumption \ref{ass:strongconcave} is equivalent to Assumption \ref{ass:model}(i) with $\lambda=\iota<0$ if $\ell$ is twice continuously differentiable.
\begin{theorem}[Theorem 3 in \cite{Caprio2025}]\label{thm:extbakryemery}Any model satisfying Assumptions~\ref{ass:model} and \ref{ass:strongconcave} satisfies the extended log-Sobolev inequality with constant $\lambda$.
\end{theorem}

In particular, for these models the extended log-Sobolev inequality always holds with a `dimension-free' constant. We now show a simple example of a model class in which this Bakry--{\'E}mery argument allows us to verify that the extended log-Sobolev inequality holds. 

\begin{example}[Bayesian hierarchical models] \label{ex:toymodel}
	Let $m\in\n$, $C_i\in\r^{d_x/m\times d_x/m},D\in\r^{d_\vartheta \times d_x/m}$. In Bayesian statistics it is common to consider hierarchical models of the form
	\begin{align*}
		Y_i = C_i X_i + U_i \qquad X_i = D \theta + V_i \qquad \text{for} \quad i=1,\dots,m,
	\end{align*}
	where $U_i$ and $V_i$ are i.i.d.~symmetric around 0 random vectors with distribution densities $p_{u,\sigma}$ and $p_{v,\tau}$.  $(\sigma,\tau,\theta)$ are parameters that might be known or not, see e.g. \cite[Section 13.4]{Gelman1995}.	We think of this model as describing noisy observations $Y_i$ of a latent process $X_i$ of interest. Sometimes, the latent variables are just parameters, and the probabilistic structure above is a way to attach a prior distribution to these. 
	The model is 
	\begin{equation*}
		\rho_\theta(x):=\prod_{i=1}^m p_{u,\sigma}(y_i-C_ix_i)p_{v,\tau}(x_i-D\theta).
	\end{equation*}
	Let us consider the setting where $\theta$ is unknown, and $(\sigma,\tau)$ are known. Within the Empirical Bayes setting, we would like to perform inference on the states using $\pi_{\theta_\star}$, with $\theta_\star$ being the maximum likelihood estimate of $\theta$, via the EM algorithm.
	In this setting, whether the extended log-Sobolev inequality holds depends solely on the noise distributions $p_u$ and $p_v$. It is common to assume that $p_u$ is a Normal density. In this case, any strongly log-concave density $p_v$ returns a model satisfying the extended log-Sobolev inequality by Theorem \ref{thm:extbakryemery}, and any gradient-Lipschitz density $p_v$ returns a model satisfying Assumption \ref{ass:gradLip}. Generally, if $p_v(x)\propto e^{-g(x)-\alpha \norm{x}^2}$, for some $\alpha>0$ and a smooth and convex $g$ (or also not too non-convex), then both the extended log-Sobolev inequality and Assumption \ref{ass:gradLip} hold. The Normal hierarchical model, corresponding to $g(x)=0$, also satisfies both requirements, although in this specific case $\pi_{\theta_\star}$ has a closed form solution.

\end{example}

\subsection{Operations preserving the extended log-Sobolev inequality and models with different completions} \label{sec:preservation}
The literature on functional inequalities contains many results which show preservation of functional inequalities under various operations. For instance, the standard log Sobolev inequality is known to be preserved under contractive mappings \citep{Bakry2014}, bounded perturbations \citep{Holley1987}, mixtures \citep{Chen2021}, convolutions with Gaussians and some class of smooth perturbations \citep{Cattiaux2022} and many more. These results extend considerably the settings where one can verify the log Sobolev inequality, showing that it goes far beyond strong log-concavity. These sorts of considerations also motivated the use of functional inequalities for the analysis of Langevin Monte Carlo \citep{Vempala2019,Chewi2023}. Here our goal is to illustrate that similar results hold for the extended log-Sobolev inequality, by generalizing the aforementioned contractive mapping and the bounded perturbation results. 

For the log Sobolev inequality, these results allow us to conclude that if a probability distribution satisfies the log Sobolev inequality, then another distribution not too different from the original, as obtained as a result of these operations, still does. To understand what the extended log-Sobolev inequality analogues say, recall that  for any given marginal likelihood, $Z_\theta$, there are many possible choices of complete likelihood, $\rho_\theta =  \pi_\theta Z_\theta$, depending on the completion, $\pi_\theta$, the choice which is known to dramatically impact the performance and convergence properties of the associated EM algorithm \citep{Dempster1977,Meng1997}. The following results then say that if a given model satisfies the extended log-Sobolev inequality, a model using another completion, that is not too different from the original, still does. As the log-Sobolev constant dictates EM's convergence speed, this in principle also gives an estimate on the performance we can expect from models obtained with different completions.

\begin{proposition}[Perturbation principle] \label{pr:perturbationprinciple}
	Let Assumption~\ref{ass:model} hold, and suppose that the measures $(\rho_\theta(d x))_{\theta\in\r^{d_\theta}}$ satisfy the extended log-Sobolev inequality with constant $\lambda>0$. Consider the measures $(\tilde{\rho}_\theta(d x))_{\theta\in\r^{d_\theta}}$ defined by $\tilde{\rho}_\theta:=\tilde{\pi}_\theta Z_\theta$, where  $\tilde{\pi}_\theta$ is a bounded perturbation of $\pi_\theta$ in the sense that $c^{-1}\leq d \pi_\theta/d \tilde{\pi}_\theta \leq c$ for some $c>1$ independent of $\theta$, and that
	\begin{align} \label{eq:extholleystroockgradgrowth}
		b:=\sup \frac{\norm{[\nabla_\theta \rho_\theta-\nabla_\theta \tilde{\rho}_\theta]/\tilde{\rho}_\theta}^2 }{[\log(Z_\star)-\log(Z_\theta)]}<\infty.
	\end{align}
	Then the measures $(\tilde{\rho}_\theta(d x))_{\theta\in\r^{d_\theta}}$ also satisfy the extended log-Sobolev inequality with constant $(\lambda-c^2 b)/2c^2$.
\end{proposition}

Notice that in the result above, the marginal likelihood $Z_\theta$ of the perturbed model is the same as that of the unperturbed one, so we are comparing models with different completions defined on a common space.  In the degenerate case in which $\pi_\theta$ is independent of $\theta$, we can take $b=0$ and this result reduces, up to a factor of 2, to the Holley--Stroock perturbation lemma \citep{Holley1987} (also see \citep[Lemma 5.1.6]{Bakry2014}) which asserts that if a probability measure $\pi_\theta$ satisfies the log Sobolev inequality with constant $\lambda$, then its bounded perturbation $\tilde{\pi}_\theta$ also does, but with constant $\lambda/c^2$. In this case, Theorem \ref{thm:extbakryemery} and Proposition \ref{pr:perturbationprinciple} immediately imply that, under Assumption~\ref{ass:model}, if the log-likelihood $\ell$ is strongly concave only in the `tails' i.e.~over $A^c \times \r^{d_\theta}$, where $A$ is some compact set in $\r^{d_x}$, then the extended log-Sobolev inequality will still hold. 

When $\pi_\theta$ depends on $\theta$, \eqref{eq:extholleystroockgradgrowth} says that the gradients of the completions need to be identical in the stationary points of the marginal likelihood (at the maximum likelihood estimate, where $\log(Z_\theta)=\log(Z_\star)$), and that outside there it should not be too large relative to the tail behaviour of the marginal likelihood.  
The marginal log-likelihood is typically intractable so verifying \eqref{eq:extholleystroockgradgrowth} for an explicit estimate of $b$ is complicated in real settings. However, we can at least use it to show that the extended log-Sobolev inequality may hold for models that are not everywhere concave.

\begin{example}[The extended log-Sobolev inequality may hold outside everywhere concave settings]
	When the perturbed $\tilde{\pi}_\theta$ only differs by $\pi_\theta$ on a compact subset of $\r^{d_x}\times \r^{d_\theta}$ that does not contain the maximum likelihood estimate, we can always satisfy the gradient growth condition \eqref{eq:extholleystroockgradgrowth}, and when such compact set is far enough from the maximum likelihood estimate, $b$ will be small enough that the constant $(\lambda-c^2 b)/c^2$ stays positive.
	
	This reasoning shows that, if we start from any model satisfying the extended log-Sobolev inequality (such as the ones considered in Example \ref{ex:toymodel}), we can always perturb it slightly around some point $(x,\theta)$ far away enough from the maximum likelihood estimate, in a way it is not-concave locally there.
	
\end{example}

Using this result, we speculate it should be possible to show that the extended log-Sobolev inequality can hold in situations where the surrogate \eqref{eq:surrogatef} is not even concave.

As is the case for the log Sobolev inequality, the extended log-Sobolev inequality is preserved under the action of Lipschitz maps. For a probability measure $\mu\in\mathcal{P}(\r^{d_x})$ and a measurable map $T$ with domain $\r^{d_x}$, let $T_{\#}\mu = \mu \circ T^{-1}$ denote the pushforward of $\mu$ by $T$. 
\begin{proposition}[Contraction principle] \label{pr:contractionprinciple}
	Let Assumption~\ref{ass:model} hold, and suppose that the measures $(\rho_\theta(d x))_{\theta\in\r^{d_\theta}}$ satisfy the extended log-Sobolev inequality with constant $\lambda>0$. If  $\tilde{\pi}_\theta$ can be written as $\tilde{\pi}_\theta=T_{\#}\pi_\theta$, for some $L_T$-Lipschitz diffeomorphism $T$,  then the measures $(\tilde{\rho}_\theta(d x))_{\theta\in\r^{d_\theta}}$ defined by $\tilde{\rho}_\theta:=\tilde{\pi}_\theta Z_\theta$ also satisfy the extended log-Sobolev inequality with constant $\lambda/\max(1,L_T^2)$.
\end{proposition}

\section{Numerical illustration} \label{sec:illustration}

Consider the setting of Example \ref{ex:toymodel} when the both noise distributions $U_i, V_i$ are Gaussians, with variance $3/4$ and $2/5$ respectively, and $m=2$, so that only two observations $\{y_1,y_2\}$, both equal to $30$, are available. The log-likelihood is equal to $\ell(\theta,x)=\sum_{i=1}^2 -\frac{3}{2}(y_i-x_i)^2-\frac{5}{4}(x_i-\theta)^2$, and the marginal maximum likelihood estimate of $\theta$ is available in closed-form, $\theta_\star=y_1$. For this model, $L_\theta=10$, $L_x=4/3+5/2$, and $L:=L_\theta\vee L_x = 10$. By Theorem \ref{thm:extbakryemery}, we can also  estimate the log-Sobolev constant with the smallest eigenvalue of $\ell$, so that $\lambda\approx 0.65$.

Figure \ref{fig:example_bounds} shows the free energy produced by the algorithms considered in this work, together with the convergence bounds derived. The left plot of the top row shows EM's free energy (continuous line), the bound implied by Proposition \ref{pr:Femconvergencetheta} (dashed line), and the improved bound of Theorem \ref{thm:Femconvergence} (dashed dotted line). Because the bound of Theorem \ref{thm:Femconvergence} entails finding an infimum across suitable sequences, which does not appear to have a closed-form solution, we used a numerical optimizer (hence, the bound displayed might not be the sharpest possible implied by the theorem). The right plot of the top row shows gradient EM's free energy (continuous line) against the bound implied by Theorem \ref{thm:grademconv} (dashed line) when considering a constant step size equal to $1/8L$. The bottom row shows free energy (continuous line) for Langevin EM, alternating gradient descent and gradient descents (again with step size $1/8L$), together with the non-asymptotic bounds (dashed lines) implied by Theorems \ref{thm:langemconvergence}, \ref{thm:cwpgdconvergence}, and \ref{thm:pgdconvergence}, respectively. All the algorithms were started in $(\theta_0,q_0)=(0,\cal{N}(0,1))$. The plots corroborate the hierarchy between speed and implementability of the algorithms we discussed in Section \ref{sec:cwpgd}, and reflected by our convergence bounds: EM is the fastest in terms of number of iterations, followed by gradient EM and Langevin EM, and then by the two gradient descents, which have comparable performance. In terms of general implementability, the order is reversed.

\begin{figure}
    			\centering
			\includegraphics[width=1\textwidth]{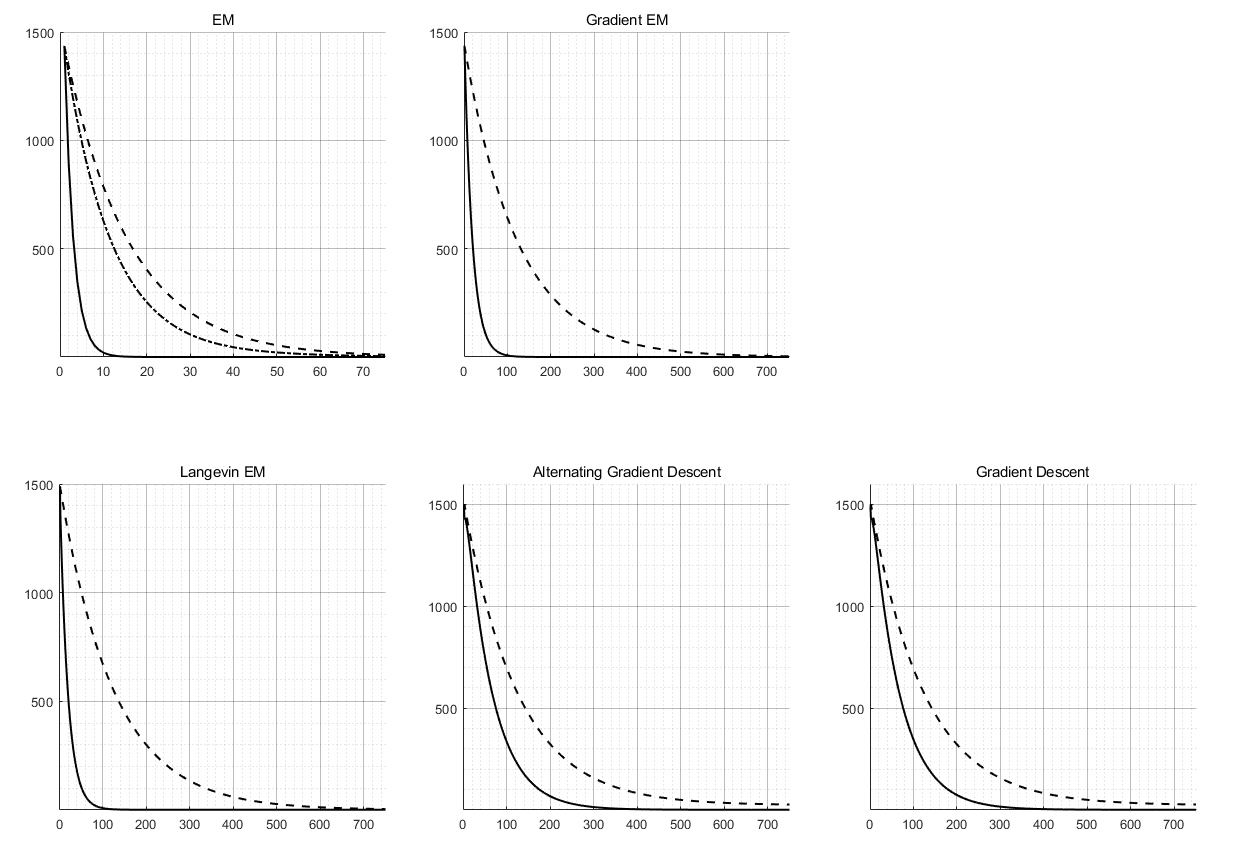}
	\caption{EM and its variants' free energy and their bounds}
	\label{fig:example_bounds}
\end{figure}

\section{Details on results from \citet{Caprio2025}} \label{sec:checkingass}

Here we justify why Theorem 4 in \citet{Caprio2025} holds under Assumptions 1(i-iv). Let $\cal{L}^1_{loc}$ denote the space of locally integrable functions, and $\cal{W}^{1,1}_{loc}$ the Sobolev space of real functions with locally integrable first weak derivatives.
\begin{proposition}
	If Assumptions 1(i-iv) hold, for any initial condition $(\theta_0,q_0)$, the flow 
	\begin{align}\label{eq:pde}
		\dot{\theta}_t = \int \nabla_\theta \ell(\theta_t,x)q_t(d x),\qquad
		\dot{q}_t=\nabla_x\cdot\bigg\{q_t \nabla_x\log\bigg(\frac{q_t}{\pi_{\theta_t}}\bigg)\bigg\}.
	\end{align}
	has a unique weak solution $(\theta_t,q_t)_{t\geq 0}$. For all $t>0$, $q_t$ admits a Lebesgue-density $t\mapsto q_t \in \cal{L}^1_{loc}((0,\infty),\cal{W}^{1,1}_{loc}(\r^{d_x}))$, and 
	\begin{equation} \label{eq:debruijintype}
		\frac{d}{dt}F(\theta_t,q_t)= -\norm{\textup{grad}_{\cal{M}_2}F(\theta_t,q_t)}^2_{\cal{M}_2} \quad \forall t>0. 
	\end{equation}
\end{proposition}
\begin{proof}
	Following an argument almost analogous to that of \citet[Theorem 11.2.1]{ambrosio2005}, the results follow from \citet[Theorem 4.0.4]{ambrosio2005} by Lemmas \ref{lemma:freenergylsc} and \ref{lemma:freenergygeoconvtype} below. 
\end{proof}

The result above is what is really needed to carry out the proofs of their Theorem 4. Notice Assumption 2 therein is only used to obtain the de-Bruijin type identity \eqref{eq:debruijintype}.

\begin{lemma} \label{lemma:freenergylsc}
	If Assumptions 1(i-iv) hold, the free energy is coercive and lower semicontinuous on $(\cal{M}_2,\mathsf{d})$.
\end{lemma}	
\begin{proof}
	Similarly to Section 11.2 in \citet{ambrosio2005}, for coercivity it is sufficient that:
	\begin{equation*}
		\exists r_\star\geq 0: \quad \inf\left\{ F(\theta,q): (\theta,q)\in\cal{M}_2, \quad \norm{\theta}^2+\int \norm{x}^2 q(dx)\leq r_\star \right\}>-\infty.
	\end{equation*} 
	Now, $F(\theta,q)=\KL(q||\pi_\theta)-\log(Z_\theta)$; $\KL(q||\pi_\theta)\geq 0$ and by assumption the second term is non-degenerate. The lower semicontinuity is shown in \citet[Lemma 18]{Caprio2025}.
\end{proof}

\begin{lemma} \label{lemma:freenergygeoconvtype}
	Let Assumption 1(i-iv) hold. For every $( \theta, q ),(\theta',q')$ there is constant speed geodesic $\gamma:[0,1]\to\cal{M}_2$ connecting $\gamma(0)=( \theta, q )$ and $\gamma(1)=(\theta',q')$ such that the free energy satisfies
	\begin{equation} \label{eq:Fgeodesicconvtype}
		F(\gamma(t))\leq (1-t)F( \theta, q )+tF(\theta',q')-\frac{\iota t(1-t)}{2}\mathsf{d}(( \theta, q ),(\theta',q'))^2.
	\end{equation}
\end{lemma}
\begin{proof}
	This follows from the same arguments of \citet[Lemma 21]{Caprio2025}.
\end{proof}

\section{Proof of Proposition \ref{pr:perturbationprinciple}} \label{app:perturbationprinciple}

Let $\phi(r):=r\log(r)$ and for a non-negative function $f$ on $\r^{d_x}\times \r^{d_\theta}$  define the  functionals
\begin{align*}
	G_\theta(f)&:= \int \phi(f)d \pi_\theta -\phi\left(\int fd \pi_\theta\right),\\
	J_\theta(f)&:= \int \frac{|\nabla_x f|^2}{f} d \pi_\theta,
\end{align*}
with $\tilde{G}_\theta(f)$ and $\Tilde{J}_\theta(f)$ denoting the corresponding functionals when the model is $\tilde{\rho}_\theta(d x):=\tilde{\pi}_\theta(d x) Z_\theta$.
Importantly, notice that if $f:=d q/d \pi_\theta$, $G_\theta(f)-\log(Z_\theta) =F(\theta,q)$ and $J_\theta(f)+\norm{\int \nabla_\theta \ell(\theta,x)q(d x)}^2=I(\theta,q)$. We begin as in the proof of \cite[Proposition 5.1.6]{Bakry2014} and use the following variational formula.
\begin{lemma}
	Letting $\phi'(r)$ denote the derivative of $\phi$ at $r$, we have,
	\begin{equation*}
		G_\theta(f) = \inf_{r>0} \int \{\phi(f)-\phi(r)-\phi'(r)(f-r)\}d \pi_{\theta}
	\end{equation*}
\end{lemma}
\begin{proof}
	Let $X\sim\pi_\theta$. Since $\phi$ is convex, for any $r>0$,
	\begin{align*}
		\phi(E\{f(X)\})-\phi(r)-\phi'(r)(E\{f(X)\}-r) \geq 0
	\end{align*}
	so that, adding $E\{\phi(f(X))\}-\phi(E\{f(X)\})$ both sides, 
	\begin{align*}
		E\{\phi(f(X))-\phi(r)-\phi'(r)(f(X)-r)\}\geq E\{\phi(f(X))\} - \phi(E\{f(X)\}) = G_\theta(f)
	\end{align*}
	with equality if and only if $r=E\{f(X)\}$.
\end{proof}
By $\phi$'s convexity $\phi(f)-\phi(r)-\phi'(r)(f-r)\geq 0$ for all $r>0$, so using the formula above we can readily write
\begin{equation*}
	\tilde{G}_\theta(f) \leq c \inf_{r>0} \int \{\phi(f)-\phi(r)-\phi'(r)(f-r)\}d \pi_{\theta}  = c G_\theta(f).
\end{equation*}
Next, we take $f:=d q/d \pi_\theta$, which gives
\begin{align*}
	\tilde{G}_\theta(f) + c(\log(Z_\star)-\log(Z_\theta)) \leq  c(G_\theta(f)+\log(Z_\star)-\log(Z_\theta)) = c(F(\theta,q)-F_\star)
\end{align*}
and because $(\rho_\theta(d x))_{\theta\in\r^{d_\theta}}$ satisfy the extended log-Sobolev inequality 
\begin{equation*}
	2\lambda\{\tilde{G}_\theta(f) + c(\log(Z_\star)-\log(Z_\theta))\} \leq cI(\theta,q) \leq c^2\int \frac{|\nabla_x f|^2}{f} d \tilde{\pi}_\theta +c\norm{\int \nabla_\theta \ell(\theta,x)q(d x)}^2.
\end{equation*}
Define now
\begin{equation*}
	\Tilde{q}(d x):=\frac{1}{\Tilde{c} }\frac{d \tilde{\pi}_\theta}{d \pi_\theta}(x) q(d x), \qquad \Tilde{c}:= \int\frac{d \tilde{\pi}_\theta}{d \pi_\theta}(x) q(d x), \qquad \Tilde{f}(x):=\frac{f}{\Tilde{c}}(x)=\frac{d \Tilde{q}}{d \tilde{\pi}_\theta}(x),
\end{equation*}
noting that $\Tilde{q}$ is a probability measure in $\Pacd$ under our assumptions. Since the functional $G_\theta$ is homogeneous, i.e.  $\tilde{G}_\theta(af)=a\tilde{G}_\theta(f)$ for all $a>0$, we may write
\begin{align} \label{eq:extholleystroockeqn1}
	2\lambda(\Tilde{F}(\theta,p)-\Tilde{F}_\star) &= 2\lambda(\tilde{G}_\theta(\Tilde{f})-\log(Z_\theta)+\log(Z_\star)) \nonumber \\
	&\leq 2\lambda\{\tilde{G}_\theta(\Tilde{f}) + \frac{c}{\Tilde{c}}(\log(Z_\star)-\log(Z_\theta))\} \nonumber \\ 
	&\leq  \frac{c^2}{\Tilde{c}}\int \frac{|\nabla_x f|^2}{f} d \tilde{\pi}_\theta + \frac{c}{\Tilde{c}}\norm{\int \nabla_\theta \ell(\theta,x)q(d x)}^2 \nonumber \\
	&\leq c^2 \bigg( \Tilde{J}_\theta(\Tilde{f}) + \norm{\int \nabla_\theta \ell(\theta,x)\frac{d \rho_\theta}{d \tilde{\rho}_\theta}(x)\Tilde{q}(d x)}^2\bigg) \nonumber \\
	&= c^2 \bigg(\Tilde{J}_\theta(\Tilde{f})+\norm{\int \frac{\nabla_\theta \rho_\theta(x)}{\tilde{\rho}_\theta(x)}\Tilde{q}(d x)}^2\bigg)
\end{align}
where in the first and penultimate inequalities we used $\Tilde{c} \leq c$ and $\log(Z_\star)-\log(Z_\theta)\geq 0$. By the inequality $(a+b)^2\leq 2a^2+2b^2$ \cite[p. 157]{Loève1977}, Jensen's inequality and the assumed growth condition on $\norm{\nabla_\theta \tilde{\rho}_\theta(x)-\nabla_\theta \rho_\theta(x)}$,
\begin{align*}
	\norm{\int \frac{\nabla_\theta \rho_\theta(x)}{\tilde{\rho}_\theta(x)}\Tilde{q}(d x)}^2 
	&\leq  2\norm{\int \nabla_\theta \Tilde{\ell}(\theta,x)\Tilde{q}(d x)}^2 + 2 \int \norm{\frac{\nabla_\theta \tilde{\rho}_\theta(x)-\nabla_\theta \rho_\theta(x)}{\tilde{\rho}_\theta(x)}}^2\Tilde{q}(d x) \\
	&\leq  2\norm{\int \nabla_\theta \Tilde{\ell}(\theta,x)\Tilde{q}(d x)}^2 + 2 b(\log(Z_\star)-\log(Z_\theta)) \\
	&\leq  2\norm{\int \nabla_\theta \Tilde{\ell}(\theta,x)\Tilde{q}(d x)}^2 + 2b (\Tilde{F}(\theta,\Tilde{q})-\Tilde{F}_\star)
\end{align*}
where in the last inequality we used $\KL(\Tilde{q}||\tilde{\pi}_\theta)\geq 0$. Combining this estimate with \eqref{eq:extholleystroockeqn1},
\begin{align*}
	2(\lambda-c^2\cdot b)(\Tilde{F}(\theta,\Tilde{q})-\Tilde{F}_\star) \leq 2c^2\bigg(\Tilde{J}_\theta(\Tilde{f})+\norm{\int \nabla_\theta \Tilde{\ell}(\theta,x)\Tilde{q}(d x)}^2\bigg) =2c^2 \Tilde{I}(\theta,\Tilde{q}).
\end{align*}

\qed

\subsection{Proof of Proposition \ref{pr:contractionprinciple}} \label{app:contractionprinciple}
Denote with $\Tilde{F}$ and $\Tilde{I}$ the free energy and extended Fisher information for the model $\tilde{\rho}_\theta$.
Consider an arbitrary $(\theta,\Tilde{q})\in\Macd$ and set $q:=T_\#\Tilde{q}$. By assumption, $\tilde{\pi}_\theta(T(x)) \det(\nabla_x T(x))=\pi_\theta(x)$ and similarly for $q$'s density, hence, by a change of variables $x\mapsto T(x)$,
\begin{align*}
	\Tilde{F}(\theta,\Tilde{q})- \log(Z_\theta) &=\int \log\bigg(\frac{\Tilde{q}(x)}{\tilde{\pi}_\theta(x)}\bigg) \Tilde{q}(x)d x 
	= 	\int \log\bigg(\frac{\Tilde{q}(T(x))}{\tilde{\pi}_\theta(T(x))}\bigg) \Tilde{q}(T(x))\det(\nabla_x T(x))d x  \\
	&= \int \log\bigg(\frac{q(x)}{\pi_\theta(x)}\bigg) q(x)d x  
	= F(\theta,q) - \log(Z_\theta). 
\end{align*}
Thus, since $q$ belongs to $\Pacd$ by the Lipschitz property of $T$, $2\lambda[\Tilde{F}(\theta,\Tilde{q})-F_\star]=2\lambda[F(\theta,q)-F_\star] \leq I(\theta,q)$ by the extended log-Sobolev inequality. On the other hand, if we let $\norm{\cdot}^2_\infty$ denote the uniform norm, 
\begin{align*}
	&I(\theta,q) \\
    &= \int \norm{\frac{\nabla_x q(x)}{q(x)}-\frac{\nabla_x \pi_\theta(x)}{\pi_\theta(x)}}^2 q(d x) + \norm{\int \left(\frac{\nabla_\theta \pi_\theta(x)}{\pi_\theta(x)} + \frac{\nabla_\theta Z_\theta}{Z_\theta}\right) q(d x)}^2 \\
	&= \int \norm{\frac{\nabla_x (\Tilde{q}(T(x)))}{\Tilde{q}(T(x))}-\frac{\nabla_x (\tilde{\pi}_\theta(T(x)))}{\tilde{\pi}_\theta(T(x))}}^2 q(d x) + \norm{\int \left(\frac{\nabla_\theta (\tilde{\pi}_\theta(T(x)))}{\tilde{\pi}_\theta(T(x))} + \frac{\nabla_\theta Z_\theta}{Z_\theta}\right) q(d x)}^2 \\
	&\leq  \norm{\nabla_x T}^2_\infty \int \norm{\frac{\nabla_{T(x)}\Tilde{q}(T(x))}{\Tilde{q}(T(x))}-\frac{\nabla_{T(x)} \tilde{\pi}_\theta(T(x))}{\tilde{\pi}_\theta(T(x))}}^2 q(d x) + \norm{\int \left(\frac{\nabla_\theta \tilde{\pi}_\theta(T(x))}{\tilde{\pi}_\theta(T(x))} + \frac{\nabla_\theta Z_\theta}{Z_\theta}\right) q(d x)}^2 \\
	&\leq \norm{\nabla_x T}^2_\infty \Tilde{I}(\Tilde{q}||\tilde{\pi}_\theta) + \norm{\int \nabla_\theta \log(\tilde{\rho}(\theta,x)) \Tilde{q}(x)d x}^2 \\
    &\leq \max(1,\norm{\nabla_x T}^2_\infty)\Tilde{I}(\theta,\Tilde{q}),
\end{align*}
where in the second line we substituted the expression for the density of $\Tilde{q}(T(x))$ and $\tilde{\pi}_\theta(T(x))$ (noting that the terms with the determinant cancel), in the third we used the chain rule and the fact that $T$ does not depend on $\theta$, and in the fourth we used the formula for $q(x)$ and a change of variables $T(x)\mapsto x$. We conclude by noting that $L_T$ provides a bound on $\norm{\nabla_x T}_\infty$.

\qed

\section{Proof of Lemma \ref{lemma:ptintegralisnotbig}} \label{app:ptintegralisnotbig}
Because of $\theta_{k+1}$'s minimality \eqref{eq:thetaisminimizing},  by Jensen's inequality and Assumption \ref{ass:gradLip},
\begin{align*}
	\norm{\int \nabla_\theta\ell(\theta_{k+1},x)p_t(d x)}^2   
	&= \norm{E\left\{\nabla_\theta\ell(\theta_{k+1},Z_t)-\nabla_\theta\ell(\theta_{k+1},X_k)\right\}}^2 
	\leq L_x^2E\left\{\norm{Z_t-X_k}^2\right\} 
\end{align*}
for any coupling of $(Z_t,X_k)$. Let $(\theta_\dagger,x_\dagger)$ denote a stationary point of $\ell$, so that $\nabla  \ell(\theta_\dagger,x_\dagger)=0$. Inserting the expression for $Z_t$, since $Z_{kh}=X_k$ a.s., using $(a+b)^2\leq 2a^2+2b^2$, the fact $h\leq 1/4L_x \Rightarrow 2L_x(t-t_-)\leq 1/2$ and again Assumption \ref{ass:gradLip}, 
\begin{align*}
	\norm{\int \nabla_\theta\ell(\theta_{k+1},x)p_t(d x)}^2 &\leq L_x^2 E\left[ E\left\{\norm{(t-t_-)\nabla_x \ell(\theta_{k+1},X_k)+\sqrt{2}(W_t-W_{t_-})}^2 \mid X_k\right\}\right] \\
	&\leq L_x(t-t_-)\left[ 2L_x(t-t_-)E\left\{\norm{\nabla_x \ell(\theta_{k+1},X_k)}^2\right\} +4d_xL_x\right] \\
	&\leq  L_x(t-t_-)\left[\frac{1}{2}E\left\{\norm{\nabla_x \ell(\theta_{k+1},X_k)}^2\right\} +4d_xL_x\right] \\
	&\leq  L_x(t-t_-)\left[\frac{L^2}{2}E\left\{\norm{X_k-x_\dagger}^2+\norm{\theta_{k+1}-\theta_\dagger}^2\right\} +4d_xL_x\right].
\end{align*}

When $q_k$ is given by an EM update, we now derive an uniform bound for  $E(\Vert X_k-x_\dagger \Vert^2+\Vert\theta_{k+1}-\theta_\dagger\Vert^2)$. Consider the point $(\theta_\star,\pi_{\theta_\star})\in\cal{M}_\star$ corresponding to the projection of $(\theta_{k+1},q_k)$ onto $\cal{M}_\star$, let $X_\star\sim \pi_{\theta_\star}$ and consider an optimal coupling between $X_\star$ and $X_k\sim q_k$. By the definition of $\mathsf{d}_{W_2}$, Theorem \ref{thm:exttalagrand}, together with the fact that the free energy can only decrease along EM iterations,
\begin{align*}
	\lambda E\left( \norm{X_k-X_\star}^2+\norm{\theta_{k+1}-\theta_\star}^2 \right) \leq 2 (F(\theta_{k+1},q_k)-F_\star) \leq 2(F(\theta_0,q_0)-F_\star). 
\end{align*}
We conclude upon combining the two previously obtained inequalities via the bounds $\norm{X_k-x_\dagger}^2\leq 2\norm{X_\star-x_\dagger}^2+2\norm{X_k-X_\star}^2$, and the analogous bound for $\norm{\theta_{k+1}-\theta_\dagger}^2$.
\qed

\section{Proof of Theorem \ref{thm:pgdconvergence}} \label{app:pgdconv}
\subsection{Descent lemma on  $\cal{M}_2$} \label{app:descentlemma}
To study Algorithm \ref{alg:pgd}, we need a generalization of the descent lemma on the whole $\cal{M}_2$. When the parameter space is trivial and equal to $\{\theta_{k+1}\}$, this result reduces to Lemma \ref{lemma:descentonP}.
\begin{lemma}[Descent lemma on $\cal{M}_2$] \label{lemma:descentonM} Let Assumptions \ref{ass:model}--\ref{ass:gradLip} hold. Let $(\vartheta_t,p_t)$ be an interpolation in $\cal{M}_2$ between $(\theta_k,q_k)$ at $t=kh$ and $(\theta_k+h\int \nabla_\theta\ell(\theta_k,x)q_k(d x),\textup{Law}(X_k+h\nabla_x \ell(\theta_{k},X_k) + \sqrt{2h}\xi_k))$, where $\textup{Law}(X_k)=q_k$, at $t=(k+1)h$. If $h\leq 1/4L$, for $t\in [kh,(k+1)h]$,
	\begin{equation*}
		\partial_t F(\theta_t,p_t) \leq - \frac{1}{2}I(\theta_t,p_t) + 6L^2d_x(t-t_-), \quad t_-:=kh.
	\end{equation*}	
\end{lemma}

To prove Lemma \ref{lemma:descentonM}, we generalise the arguments of \citet{Vempala2019}, and the refinements in \citet[Section 4.2]{Chewi2023} under the hypothesis of non-trivial parameter space, and our proof is, therefore, somewhat similar to theirs.
Consider the continuous time process given for all $t \in [kh,(k+1)h)$ and any $k \in \mathbb{N}_0$ by
\begin{equation} \label{eq:pgdinterp}
	\begin{aligned}
		\vartheta_{kh} =& \theta_k & d \vartheta_t &= \int \nabla_\theta \ell(\vartheta_{t_-},x)p_{t_-}(d x) d t,\\
		Z_{kh} =& X_k & d Z_t &= \nabla_x \ell(\vartheta_{t_-},Z_{t_-}) d t + \sqrt{2} d W_t, 
	\end{aligned}
\end{equation}
where we set $t_-:=k h, p_t:=\textup{Law}(Z_t)$ and where $W_t$ is a Brownian Motion. Notice that this is an interpolation of the Algorithm \ref{alg:pgd} iterates. To prove Lemma \ref{lemma:descentonM} we need three auxiliary results.
\begin{lemma} \label{lemma:approxpde}
	If Assumptions \ref{ass:model}--\ref{ass:gradLip} hold, $\{(\vartheta_t,p_t): t_- \leq t < t_-+h\}$ is a solution of the PDE 
	\begin{equation}	  \label{eq:interpde}
		\begin{aligned}
			\partial_t\vartheta_t &= \int \nabla_\theta \ell(\vartheta_{t},x)p_{t}(d x) + E\{\nabla_\theta \ell(\vartheta_{t_-},Z_{t_-})-\nabla_\theta \ell(\vartheta_{t},Z_t)\}\\
			\partial_tp_t&=\nabla_x\cdot\bigg[p_t \nabla_x\log\bigg(\frac{p_t}{\rho_{\vartheta_t}}\bigg)+p_tE\{\nabla_x \ell(\vartheta_{t},\cdot)-\nabla_x \ell(\vartheta_{t_-},Z_{t_-})\mid Z_t=\cdot\}\bigg].
		\end{aligned}
	\end{equation}
\end{lemma}
\begin{proof}
	Conditionally on $(\vartheta_{t_-},Z_{t_-})=(\vartheta,z)$, $\{(\vartheta_t,Z_t):t_-\leq t< t_-+h\}$ has a time independent drift, hence $p_{t|t_-}(\cdot|z)$, the conditional density of $Z_t$ given $Z_{t_-}=z$, is the density of a Normal with  mean $z+(t-t_-)\nabla_x \ell(\vartheta,z)$ and variance $2(t-t_-)$.  Hence, we can directly check that $p_{t|t_-}$ solves the following PDE:
	\begin{align*}
		\partial_t \vartheta_t &= \int \nabla_\theta \ell(\vartheta_{t_-},x)p_{t_-}(d x), \\
		\partial_t p_{t|t_-}&=\nabla_x\cdot\bigg\{p_{t|t_-} \nabla_x\log\bigg(\frac{p_{t|t_-}}{\rho_{\vartheta_{t_-}}}\bigg)\bigg\} = \Delta_x p_{t|t_-} - \nabla_x \cdot \{p_{t|t_-}\nabla_x\ell(\vartheta_{t_-},Z_{t_-})\}.
	\end{align*}
	One can also derive the above immediately by using formul\ae\ for the Fokker--Planck equation of It\^{o} diffusions, see \citet[Chapters 2 and 4]{Pavliotis2014}. We notice that all the derivatives of $x\mapsto p_{t|t_-}(x|z)$ exist and are continuous in all arguments by Assumption \ref{ass:gradLip}, that  $z\mapsto p_{t|t_-}(x|z)$ is at least continuous for all $x\in\r^{d_x}$, and that $t\mapsto p_{t|t_-}(x|z)$ is continuously differentiable for all $x,z\in\r^{d_x}$. By Bayes' rule, if we take expectation w.r.t.~$p_{t_-}$,
	\begin{align*}
		&E\left\{p_{t|t_-}(x|Z_{t_-})\nabla_x\ell(\vartheta_{t_-},Z_{t_-})\right\} 
		= \int p_{t|t_-}(x|z)\nabla_x\ell(\vartheta_{t_-},z)p_{t_-}(d z) \\ 
		&= p_{t}(x)\int p_{t_-|t}(d z|x)\nabla_x\ell(\vartheta_{t_-},z) 
		= p_{t}(x)E\left\{\nabla_x\ell(\vartheta_{t_-},Z_{t_-})\mid Z_t=x\right\}.
	\end{align*}
	Hence,  taking expectation in the Fokker--Planck equation, the regularity properties of $p_{t|t_-}$ allow us to use the tower rule and compute
	\begin{align*}
		\partial_t \vartheta_t = \int \nabla_\theta \ell(\vartheta_{t_-},x)p_{t_-}(d x),\quad
		\partial_t p_{t}= \Delta_x p_{t} - \nabla_x \cdot [p_{t}E\{\nabla_x\ell(\vartheta_{t_-},Z_{t_-})\mid Z_t=\cdot\}].
	\end{align*}
	Adding and subtracting $E\{\nabla_\theta \ell(\vartheta_{t},Z_t)\}$ in the first expression and $\nabla_x\cdot[p_{t}E\{\nabla_x\ell(\vartheta_t,\cdot)|Z_t=\cdot\}]=\nabla_x\cdot\{p_{t}\nabla_x \log(\rho_{\vartheta_t}\})$ in the second yields the claim. 
\end{proof}
\begin{lemma} \label{lemma:debruijinapproxineq}
	If Assumptions \ref{ass:model}--\ref{ass:gradLip} hold, for all $t\in[t_-,t_-+h)$, 
	\begin{align} \label{eq:debruijinapproxineq}
		\partial_t F(\vartheta_t,p_t) \leq - \frac{3}{4}I(\vartheta_t,p_t) + E\left\{\norm{\nabla \ell(\vartheta_{t},Z_{t})-\nabla \ell(\vartheta_{t_-},Z_{t_-})}^2\right\}
	\end{align}
\end{lemma}
\begin{proof}		
	By the regularity properties of $p_{t|t_-}$ illustrated in the proof of the previous lemma, $p_t(x) = E\{p_{t|t_-}(x,Z_{t_-})\}$ is smooth in $x$ and at least continuously differentiable in $t$. This allows us to compute
	\begin{align} \label{eq:debruijinapproxineq_1}
		\partial_t F(\vartheta_t,p_t) 
		&=\int \bigg\{\log\left(\frac{p_t(x)}{\rho_{\vartheta_t}(x)}\right)+1\bigg\} \nabla_x\cdot \bigg\{p_t(x)\nabla_x\log\bigg(\frac{p_t(x)}{\rho_{\vartheta_t}(x)}\bigg)\bigg\}d x - \norm{\int \nabla_\theta \ell(\vartheta_t,x)p_t(d x)}^2 \nonumber \\
		&+ \int   \bigg\{\log\left(\frac{p_t(x)}{\rho_{\vartheta_t}(x)}\right)+1\bigg\} \nabla_x\cdot [p_t(x)E\{\nabla_x \ell(\vartheta_{t},x)-\nabla_x \ell(\vartheta_{t_-},Z_{t_-})\mid Z_t=x\}]d x \nonumber \\
		& + E\{\nabla_\theta \ell(\vartheta_{t},Z_t)-\nabla_\theta \ell(\vartheta_{t_-},Z_{t_-})\}\cdot E\{ \nabla_\theta\ell(\vartheta_t,Z_{t})\} \nonumber \\
		&= -I(\vartheta_t,p_t) + \int  \nabla_x \log\left(\frac{p_t(x)}{\rho_{\vartheta_t}(x)}\right)\cdot E\{\nabla_x \ell(\vartheta_{t_-},Z_{t_-})-\nabla_x \ell(\vartheta_{t},x)\mid Z_t=x\} p_t(d x) \nonumber \\
		&+ E\{\nabla_\theta \ell(\vartheta_{t},Z_t)-\nabla_\theta \ell(\vartheta_{t_-},Z_{t_-})\}\cdot E\{ \nabla_\theta\ell(\vartheta_t,Z_{t})\},
	\end{align}
	where we used the regularity of $\{(\vartheta_t,p_t):t_-\leq t < t_- + h\}$ to differentiate inside the integral and the chain rule in the first equality, and we integrated by parts for the second.  Now, by the Young's inequality $ab\leq (a^2/4) + b^2$ and then Jensen's inequality, we have 
	\begin{align*}
		&\int  \nabla_x \log\left(\frac{p_t(x)}{\rho_{\vartheta_t}(x)}\right)\cdot E\{\nabla_x \ell(\vartheta_{t_-},Z_{t_-})-\nabla_x \ell(\vartheta_{t},x)\mid Z_t=x\} p_t(d x)  \\
		&\leq \frac{1}{4}I(p_t||\pi_{\vartheta_{t}}) + \int \norm{E\{\nabla_x \ell(\vartheta_{t_-},Z_{t_-})-\nabla_x \ell(\vartheta_{t},x)\mid Z_t=x\}}^2p_t(d x)  \\
		&\leq \frac{1}{4}I(p_t||\pi_{\vartheta_{t}}) + E\left\{\norm{\nabla_x \ell(\vartheta_{t},Z_t)-\nabla_x \ell(\vartheta_{t_-},Z_{t_-})}^2\right\}, 
	\end{align*} 
	and similarly
	\begin{align*}
		&E\{\nabla_\theta \ell(\vartheta_{t},Z_t)-\nabla_\theta \ell(\vartheta_{t_-},Z_{t_-})\}\cdot E\{ \nabla_\theta\ell(\vartheta_t,Z_{t})\} \\
		&\leq \frac{1}{4}\norm{\int \nabla_\theta \ell(\vartheta_{t},x)p_t(d x)}^2 + 	E\left\{\norm{\nabla_\theta \ell(\vartheta_{t},Z_t)-\nabla_\theta \ell(\vartheta_{t_-},Z_{t_-})}^2\right\},
	\end{align*}
	now by combining these last two estimates with \eqref{eq:debruijinapproxineq_1} we prove the desired result. 
\end{proof}	

Our goal now is to obtain a bound on the rightmost term in \eqref{eq:debruijinapproxineq} in terms of $I$.

\begin{lemma} \label{lemma:sqnormgradellbound}
	If Assumption \ref{ass:gradLip} holds, for any $\theta\in\r^{d_\theta}$ and $q\in\cal{P}_2(\r^{d_x})$ we have
	\begin{align*}
		\norm{E\left\{\nabla_\theta \ell(\theta,X)\right\}}^2+E\left\{\norm{\nabla_x \ell(\theta,X)}^2\right\} \leq I(\theta,q) + 2Ld_x, \quad \text{where} \quad \textup{Law}(X)=q.
	\end{align*}
\end{lemma}
\begin{proof}
	This inequality follows almost immediately from \citet[Lemma 4.2.5]{Chewi2023} or \citet[Lemma 16]{Chewi2022} and the definition of $I(\theta,q)$. We reproduce those proofs here in our notation for convenience. Consider the (overdamped) Langevin diffusion with stationary distribution $\pi_\theta\propto e^{\ell(\theta,\cdot)}$. Its generator satisfies $\cal{L}\ell(\theta,\cdot)=\Delta_x \ell(\theta,\cdot) + \norm{\nabla_x \ell(\theta;\cdot)}^2$. We can estimate directly
	\begin{align*}
		E\left\{\norm{\nabla_x \ell(\theta,X)}^2\right\} 
		&= E\{-\Delta_x \ell(\theta,X)+\cal{L}\ell(\theta,X)\} \\
		&\leq Ld_x + \int \cal{L}\ell(\theta;x) \frac{d q}{d \pi_\theta}(x)\pi_\theta(d x)  \\
		&= Ld_x + \int \norm{\nabla_x\ell(\theta,x)}^2 \frac{d q}{d \pi_\theta}(x) \pi_\theta( x) + \Delta_x \ell(\theta,x) \frac{d q}{d \pi_\theta}(x) \pi_\theta(x) d x \\
		&=   Ld_x + \int \norm{\nabla_x\ell(\theta,x)}^2 \frac{d q}{d \pi_\theta}(x) \pi_\theta(x) - \iprod{\nabla_x\ell(\theta,x)}{\nabla_x\left(\frac{d q}{d \pi_\theta}(x)\pi_\theta(x)\right)}d x,  
	\end{align*}
	where in the second line we used $-\Delta_x \ell(\theta,\cdot) \leq L d_x$ by Assumption \ref{ass:gradLip}, in the fourth integration by parts. Now, with the product rule and the fact $\nabla_x \ell(\theta,x) \pi_\theta(x)=\nabla_x \pi_\theta(x)$, 
	\begin{align*}
		E\left\{\norm{\nabla_x \ell(\theta,X)}^2\right\} 
		&=Ld_x - \int \iprod{\nabla_x\ell(\theta,x)}{\nabla_x\frac{d q}{d \pi_\theta}(x)}\pi_\theta(d x)\\
		&= Ld_x - 2\int \iprod{\sqrt{\frac{d q}{d \pi_\theta}}(x)\nabla_x\ell(\theta,x)}{\nabla_x\sqrt{\frac{d q}{d \pi_\theta}}(x)}\pi_\theta(d x) \\
		&\leq Ld_x + \frac{1}{2}E\left\{\norm{\nabla_x \ell(\theta,X)}^2\right\} + 2\int \bigg\Vert\nabla_x\sqrt{\frac{d q}{d \pi_\theta}}(x)\bigg\Vert^2 \pi_\theta(d x) \\
		&= Ld_x + \frac{1}{2}E\left\{\norm{\nabla_x \ell(\theta,X)}^2\right\} + \frac{1}{2}I(q||\pi_\theta),
	\end{align*}
	where we used the chain rule, and the Young's inequality $ab\leq (a^2/4) + b^2$ again. Now re-arranging, adding $\big\Vert E\{\nabla_\theta \ell(\theta,x)\}\big\Vert^2$ to both sides and using the definition of  $I(\theta,q)$ proves the bound.
\end{proof}	

Having established these intermediate results we can return to the now straightforward proof of the lemma of interest.
\begin{proof}[Proof of Lemma \ref{lemma:descentonM}]
	By Assumption \ref{ass:gradLip}, the inequality $(a+b)^2\leq 2a^2 +2b^2$, and Jensen's, 
	\begin{align*}
		&E\left\{\norm{\nabla \ell(\vartheta_{t},Z_{t})-\nabla \ell(\vartheta_{t_-},Z_{t_-})}^2 \right\} 
		\leq L^2 E\left\{\norm{(\vartheta_{t},Z_{t})-(\vartheta_{t_-},Z_{t_-})}^2\right\} \\
		&= L^2(t-t_-)^2 \left[\norm{E\{\nabla_\theta\ell(\vartheta_{t_-},Z_{t_-})\}}^2+E\left\{\norm{\nabla_x \ell(\vartheta_{t_-},Z_{t_-})}^2\right\} \right] + 2L^2E\left\{\norm{W_t-W_{t_-}}^2\right\} \\
		&\leq 2L^2(t-t_-)^2\left[\norm{E\{\nabla_\theta\ell(\vartheta_{t},Z_{t})\}}^2+ E\left\{\norm{\nabla_x \ell(\vartheta_{t},Z_{t})}^2\right\} + E\left\{\norm{\nabla \ell(\vartheta_{t},Z_{t})-\nabla \ell(\vartheta_{t_-},Z_{t_-})}^2 \right\} \right] \\
		&+ 2L^2E\left\{\norm{W_t-W_{t_-}}^2\right\}.
	\end{align*}
	Hence, if $h\leq 1/2L\Rightarrow 1/2\leq (1-2L^2(t-t_-)^2)$ we can then rearrange into
	\begin{align*}
		&\frac{1}{2}E\left\{\norm{\nabla \ell(\vartheta_{t},Z_{t})-\nabla \ell(\vartheta_{t_-},Z_{t_-})}^2 \right\} \leq (1-2L^2(t-t_-)^2)E\left\{\norm{\nabla \ell(\vartheta_{t},Z_{t})-\nabla \ell(\vartheta_{t_-},Z_{t_-})}^2 \right\} \\
		&\leq 2L^2(t-t_-)^2\left[\norm{E\left\{\nabla_\theta\ell(\vartheta_{t},Z_{t})\right\}}^2+E\left\{\norm{\nabla_x \ell(\vartheta_{t},Z_{t})}^2\right\}\right] + 2L^2d_x(t-t_-).
	\end{align*}
	If, further, $h\leq 1/4L\Rightarrow 4L^2(t-t_-)^2\leq 1/4$, we can use Lemma \ref{lemma:sqnormgradellbound} to estimate 
	\begin{align*}
		E\left\{\norm{\nabla \ell(\vartheta_{t},Z_{t})-\nabla \ell(\vartheta_{t_-},Z_{t_-})}^2 \right\}  &\leq \frac{1}{4}I(\vartheta_{t},p_{t}) + 8L^3d_x(t-t_-)^2 + 4L^2d_x(t-t_-) \\
		&\leq \frac{1}{4}I(\vartheta_{t},p_{t}) + 6L^2d_x(t-t_-),
	\end{align*}
	where we used $8L(t-t_-)\leq 2$. We conclude by combining this inequality with Lemma~\ref{lemma:debruijinapproxineq}.
\end{proof}

\subsection{Proof of Theorem \ref{thm:pgdconvergence}} 
By the descent Lemma \ref{lemma:descentonM} and the extended log-Sobolev inequality we have
\begin{align*}
	\partial_t (F(\vartheta_t,p_t)-F_\star) \leq - \frac{1}{2}I(\vartheta_t,p_t) + 6L^2d_x(t-t_-)  \leq -\lambda (F(\vartheta_t,p_t)-F_\star)  + 6L^2d_x(t-t_-).
\end{align*}
Now we proceeding similarly to the proof of Theorem \ref{thm:Femconvergence}, and write 
\begin{align*}
	\partial_t \left\{e^{t\lambda}(F(\vartheta_t,p_t)-F_\star) \right\} \leq 6L^2d_xhe^{t\lambda},
\end{align*}
so that integrating from $t=kh$ to $t=(k+1)k$ gives
\begin{align*}
	(F(\vartheta_{(k+1)h},p_{(k+1)h})-F_\star)  \leq e^{-h\lambda} (F(\vartheta_{kh},p_{kh})-F_\star) + 6L^2d_xh^2. 
\end{align*}
We conclude by substituing $h$ with $h_{k+1}$ and iterating the above inequality, and by combining the result with the extension of Talagrand inequality via Theorem \ref{thm:exttalagrand}.
\qed

\section{Proof of Theorem \ref{thm:cwpgdconvergence}}

The proof of this result is similar in spirit to that of Theorem~\ref{thm:pgdconvergence}. Let $(p_t)$ be an interpolation in $\cal{P}_2(\r^{d_x})$ between $q_k$ at $t=kh$ and $q_{k+1}=\textup{Law}(X_k+h\nabla_x \ell(\theta_{k+1},X_k) + \sqrt{2h}\xi_k)$, where $\textup{Law}(X_k)=q_k$, at $t=(k+1)h$, defined by the law of
\begin{align*}
	Z_{t_-} =& X_k & d Z_t &= \nabla_x \ell(\theta_{k+1},Z_{t_-}) d t + \sqrt{2} d W_t,
\end{align*}
where $t\in [kh,(k+1)h]$ and $t_-=kh$. Let $(\vartheta_t)$ instead be an interpolation in $\r^{d_\theta}$ between $\theta_k$ at $t=kh$ and $\theta_{k+1}$ at $t=(k+1)h$, defined by 
\begin{align*}
	\vartheta_{kh}&=\theta_k &  d\vartheta_t&=\int \nabla_\theta\ell(\theta_k,x)q_k(dx) dt.
\end{align*}
Notice that, as opposed to the proof of Theorem~\ref{thm:pgdconvergence}, here $(p_t)$ interpolates the iterates of Algorithm~\ref{alg:cwpgd}.
\begin{lemma} \label{lemma:cwdebruijinapproxineq2}
	If Assumptions \ref{ass:model}--\ref{ass:gradLip} hold, for all $t\in[t_-,t_-+h)$, 
	\begin{equation}	  \label{eq:cwinterpde}
		\begin{aligned}
			\partial_t F(\vartheta_{t},p_t) =& - E\{\nabla_\theta \ell(\vartheta_t,Z_t)\}\cdot E\{\nabla_\theta \ell(\theta_k,X_k)\} \\
			&- \int \nabla_x \log\left(\frac{p_t(x)}{\pi_{\vartheta_{t}(x)}} \right) \cdot \left\{ \nabla_x \log(p_t(x)) - E\{\nabla_x \ell(\theta_{k+1},X_k)|Z_t=x\}\right\}p_t(dx).
		\end{aligned}
	\end{equation}
\end{lemma}
\begin{proof}
	In $t\in[t_-,t_-+h)$, $p_t$ satisfies
	\begin{align*}
		\partial_t p_t &= \Delta_x p_t - \nabla_x \cdot (p_t E\{\nabla_x\ell(\theta_{k+1},X_k)|Z_t=\cdot\} ) \\
		&= \nabla_x \cdot (p_t(\nabla_x \log(p_t)-E\{\nabla_x\ell(\theta_{k+1},X_k)|Z_t=\cdot\}))
	\end{align*}	
	(this is shown in at the end of Lemma~\ref{lemma:approxpde}, here we just have $\theta_{k+1}$ instead of $\vartheta_{t_-}=\theta_k$, or can be deduced from the final line in the proof of \citet[Proposition 4.2.3]{Chewi2023}). 
	Differentiating $F(\vartheta_t,p_t)$,
	\begin{align*}
		\partial_t F(\vartheta_t,p_t) = \int \left(1+\log \left(\frac{p_t(x)}{\pi_{\vartheta_t}(x)} \right) \right)\partial_t p_t - \int \partial_t \ell(\vartheta_t,x)p_t(dx)
	\end{align*}
	and we conclude by integrating by parts.
\end{proof}

We will work towards bounding the right hand side of \eqref{eq:cwinterpde} with $I(\vartheta_t,p_t)$, similarly to the proof of Lemma \ref{lemma:debruijinapproxineq}.
The first term can be decomposed, by Young's inequality, via
\begin{align*} 
	&- E\{\nabla_\theta \ell(\vartheta_t,Z_t)\}\cdot E\{\nabla_\theta \ell(\theta_k,X_k)\} \nonumber \\
	&=-\norm{E\{\nabla_\theta \ell(\vartheta_t,Z_t)\}}^2 - E\{\nabla_\theta \ell(\vartheta_t,Z_t)\}\cdot \left[ E\{\nabla_\theta \ell(\vartheta_t,Z_t)\}-E\{\nabla_\theta \ell(\theta_k,X_k)\}\right] \nonumber \\
	&\leq -\norm{E\{\nabla_\theta \ell(\vartheta_t,Z_t)\}}^2 + \frac{1}{4}\norm{E\{\nabla_\theta \ell(\vartheta_t,Z_t)\}}^2 + E\left\{\norm{\nabla_\theta\ell(\vartheta_t,Z_t)-\nabla_\theta\ell(\theta_k,X_k)}^2\right\}
\end{align*}
For the second, similarly, 
\begin{align*}
	&-\int \nabla_x \log\left(\frac{p_t(x)}{\pi_{\vartheta_{t}(x)}} \right)\cdot \left\{ \nabla_x \log(p_t(x)) - E\{\nabla_x \ell(\theta_{k+1},X_k)|Z_t=x\}\right\}p_t(dx)  \nonumber \\
	&\leq -I(p_t||\pi_{\vartheta_t}) + \frac{1}{4}I(p_t||\pi_{\vartheta_t}) + \int E\left\{\norm{\nabla_x\ell(\vartheta_t,x)-\nabla_x\ell(\theta_{k+1},X_k)}^2|Z_t=x\right\}p_t(dx)\nonumber \nonumber \\
	&\leq -\frac{3}{4}I(p_t||\pi_{\vartheta_t}) + E\left\{\norm{\nabla_x\ell(\vartheta_t,Z_t)-\nabla_x\ell(\theta_{k+1},X_k)}^2\right\}.
\end{align*}
Adding these, we find an analogue of the result of Lemma \ref{lemma:debruijinapproxineq}:
\begin{align*} \label{eq:Fvariationssoul1}
	&\partial_t F(\vartheta_t,p_t) \\
	&\leq - \frac{3}{4}I(\theta_t,p_t) + E\{\norm{\nabla_x\ell(\vartheta_t,Z_t)-\nabla_x\ell(\theta_{k+1},X_k)}^2\} + E\{\norm{\nabla_\theta\ell(\vartheta_t,Z_t)-\nabla_\theta\ell(\theta_k,X_k)}^2\}.
\end{align*}
Next, under Assumption \ref{ass:gradLip}, and much as in the proof of Lemma \ref{lemma:descentonM}, we directly estimate
\begin{align*}
	&E\left\{\norm{\nabla_x\ell(\vartheta_t,Z_t)-\nabla_x\ell(\theta_{k+1},X_k)}^2\right\} + E\left\{\norm{\nabla_\theta\ell(\vartheta_t,Z_t)-\nabla_\theta\ell(\theta_k,X_k)}^2\right\} \\
	&\leq
	2E\left\{\norm{\nabla\ell(\vartheta_t,Z_t)-\nabla\ell(\theta_{k},X_k)}^2\right\} +2E\left\{\norm{\nabla_x\ell(\theta_k,X_k)-\nabla_x\ell(\theta_{k+1},X_k)}^2\right\} \\
	&\leq 
	2L^2E\left\{\norm{(\vartheta_t,Z_t)-(\theta_k,X_k)}^2\right\} + 2L^2\norm{\theta_k-\theta_{k+1}}^2 \\
	&=2L^2(t-t_-)^2\{\norm{E\left\{\nabla_\theta \ell(\theta_k,X_k)\right\}}^2 + E\left\{\norm{\nabla_x\ell(\theta_{k+1},X_k)}^2\}\right\} \\
	&+4L^2(t-t_-)d_x + 2L^2h^2\norm{E\{\nabla_\theta \ell(\theta_k,X_k)\}}^2 \\
	&\leq 4L^2h^2\{\norm{E\left\{\nabla_\theta \ell(\theta_k,X_k)\right\}}^2 + E\left\{\norm{\nabla_x\ell(\theta_{k+1},X_k)}^2\}\right\}+4L^2hd_x \\
	&\leq 8L^2h^2\big[\norm{E\left\{\nabla_\theta \ell(\vartheta_t,Z_t)\right\}}^2 + E\left\{\norm{\nabla_x\ell(\vartheta_t,Z_t)}^2\right\} \\
	&+E\left\{\norm{\nabla_x\ell(\vartheta_t;Z_t)-\nabla_x\ell(\theta_{k+1},X_k)}^2\right\} 
	+ E\left\{\norm{\nabla_\theta\ell(\vartheta_t,Z_t)-\nabla_\theta\ell(\theta_k,X_k)}^2\right\}\big] +4L^2hd_x.
\end{align*}
Re-arranging, and using $h\leq 1/4L\Rightarrow 8h^2L^2\leq 1/2\Rightarrow (1-8h^2L^2)\geq 1/2$,
\begin{align*}
	&\frac{1}{2}\big[E\left\{\norm{\nabla_x\ell(\vartheta_t,Z_t)-\nabla_x\ell(\theta_{k+1},X_k)}^2\right\} 
	+ E\left\{\norm{\nabla_\theta\ell(\vartheta_t,Z_t)-\nabla_\theta\ell(\theta_k,X_k)}^2\right\}\big] \\
	&\leq  8L^2h^2\big[\norm{E\{\nabla_\theta \ell(\vartheta_t,Z_t\}}^2 + E\left\{\norm{\nabla_x\ell(\vartheta_t,Z_t}^2\right\} \big] +4L^2hd_x \\
	&\leq 8L^2h^2 \big[I(\vartheta_t,p_t)+2Ld_x\big] +4L^2hd_x,
\end{align*}
where we also used Lemma \ref{lemma:sqnormgradellbound}. Since $h\leq 1/8L\Rightarrow 16h^2L^2 \leq 1/4$ as well, combining with our previous bound for $	\partial_t F(\vartheta_t,p_t)$,
\begin{align*}
	\partial_t F(\vartheta_t,p_t) \leq -\frac{1}{2}I(\vartheta_t,p_t)+ 32L^3h^2d_x + 8L^2hd_x \leq -\frac{1}{2}I(\vartheta_t,p_t)+ 12L^2hd_x.
\end{align*}
By the extended log-Sobolev inequality, we obtain, 
\begin{align*}
	\partial_t F(\vartheta_t,p_t) \leq -\lambda F(\vartheta_t,p_t)+ 12L^2hd_x,
\end{align*}
at which point we use the same arguments as the proof of Theorem \ref{thm:pgdconvergence}.

\end{document}